\documentclass[10pt,journal,compsoc]{IEEEtran}
\ifCLASSOPTIONcompsoc
  \usepackage[nocompress]{cite}
\else
  \usepackage{cite}
\fi
\ifCLASSINFOpdf
\else
\fi
\usepackage{epsfig}
\usepackage{helvet}
\usepackage{courier}
\usepackage{floatrow}
\usepackage{multirow}
\newfloatcommand{capbtabbox}{table}[][\FBwidth]
\usepackage{comment}
\usepackage[utf8]{inputenc}
\usepackage[T1]{fontenc}
\usepackage{amsfonts}
\usepackage{units}
\usepackage{microtype}
\usepackage[dvipsnames]{xcolor}
\usepackage{soul}
\usepackage{url}
\usepackage[small]{caption}
\usepackage{graphicx}
\usepackage{amsmath}
\usepackage{booktabs}
\usepackage{epstopdf}
\usepackage{amsthm}
\usepackage{amssymb}
\usepackage{subfigure}
\usepackage{wrapfig}
\newcommand{\tabincell}[2]{\begin{tabular}{@{}#1@{}}#2\end{tabular}}

\newtheorem{myTheo}{Theorem}

\newtheorem{myLem}{Lemma}
\newtheorem{myDef}{Definition}
\newtheorem{thm}{Theorem}[section]
\newtheorem{lem}{Lemma}[section]

\usepackage{bbding}
\usepackage{enumitem}
\usepackage{float}
\usepackage{multicol}

\usepackage[ruled,vlined]{algorithm2e}%[ruled,vlined]{
\usepackage{algpseudocode}
\renewcommand{\algorithmicrequire}{\textbf{Input:}} 
\renewcommand{\algorithmicensure}{\textbf{Output:}}

\begin{document}
\title{Diverse Sample Generation: Pushing the Limit of Generative Data-free Quantization}

\author{Haotong~Qin, Yifu~Ding, Xiangguo~Zhang, Jiakai~Wang,
Xianglong~Liu$^*$,~\IEEEmembership{Member,~IEEE,} and~Jiwen~Lu,~\IEEEmembership{Senior~Member,~IEEE}
\IEEEcompsocitemizethanks{\IEEEcompsocthanksitem H. Qin, Y. Ding, X. Zhang, and X. Liu (corresponding author, E-mail: xlliu@buaa.edu.cn) are with the State Key Laboratory of Software Development Environment, Beihang University, China.
\IEEEcompsocthanksitem J. Wang is with Zhongguancun Laboratory, China.
\IEEEcompsocthanksitem J. Lu is with the Department of Automation, Tsinghua University, China.
}
% \author{Anonymity
\thanks{Our code is released at \protect\url{https://github.com/htqin/DSG}.}}

\markboth{Journal of \LaTeX\ Class Files,~Vol.~14, No.~8, August~2015}%
{Shell \MakeLowercase{\textit{et al.}}: Bare Demo of IEEEtran.cls for Computer Society Journals}



\IEEEtitleabstractindextext{%
\begin{abstract}
Generative data-free quantization emerges as a practical compression approach that quantizes deep neural networks to low bit-width without accessing the real data. This approach generates data utilizing batch normalization (BN) statistics of the full-precision networks to quantize the networks. However, it always faces the serious challenges of accuracy degradation in practice. We first give a theoretical analysis that the diversity of synthetic samples is crucial for the data-free quantization, while in existing approaches, the synthetic data completely constrained by BN statistics experimentally exhibit severe homogenization at distribution and sample levels. This paper presents a generic \textbf{D}iverse \textbf{S}ample \textbf{G}eneration (\textbf{DSG}) scheme for the generative data-free quantization, to mitigate detrimental homogenization. We first slack the statistics alignment for features in the BN layer to relax the distribution constraint. Then, we strengthen the loss impact of the specific BN layers for different samples and inhibit the correlation among samples in the generation process, to diversify samples from the statistical and spatial perspectives, respectively. Comprehensive experiments show that for large-scale image classification tasks, our DSG can consistently quantization performance on different neural architectures, especially under ultra-low bit-width. And data diversification caused by our DSG brings a general gain to various quantization-aware training and post-training quantization approaches, demonstrating its generality and effectiveness.
%\thanks{Our code is released at \url{https://github.com/htqin/DSG}.}
\end{abstract}

}

\maketitle

\IEEEdisplaynontitleabstractindextext

% \fi
\IEEEpeerreviewmaketitle

%%%%%%%%% BODY TEXT
\IEEEraisesectionheading{\section{Introduction}}
\label{sec:intro}
\IEEEPARstart{W}{ith} the advent of deep learning, the deep neural network has achieved a great success in a variety of fields, such as image classification~\cite{krizhevsky2012imagenet,VeryDeepConvolutional,9384353}, object detection~\cite{DBLP:journals/corr/GirshickDDM13,DBLP:journals/corr/Girshick15,DBLP:journals/corr/abs-1904-02701,NIPS2015_5638}, semantic segmentation~\cite{Everingham:2010:PVO:1747084.1747104,Zhuang_2019_CVPR}, etc.
Nevertheless, it is still a significant challenge to apply advanced neural networks on resource-limited devices for their high memory usage and expensive computation. With more and more hardware supporting low bit-width computations, network quantization emerges as an efficient method to compress and accelerate models~\cite{8444745,9072484,8417979,9319565,Qin_2020_CVPR,8674614,9454278,8573867,Qin_2020_pr}.
Many quantization methods, called quantization-aware training (QAT), apply the following pipeline: considering the quantization function in the training process on the original dataset, and minimizing the loss caused by the quantization through backward propagation. Since QAT methods require the finetuning steps, it is considered to be time-consuming and computationally intensive~\cite{gupta2015deep,jacob2018quantization,qin2020bipointnet}.
Thus, quantization without training or finetuning process is also demanded in the industry, which is called post-training quantization (PTQ) in recent studies~\cite{banner2019posttraining,choukroun2019low,zhao2019improving,nagel2020down,li2021brecq}.

\begin{figure}[t]
\flushleft
\includegraphics[width=0.95\linewidth]{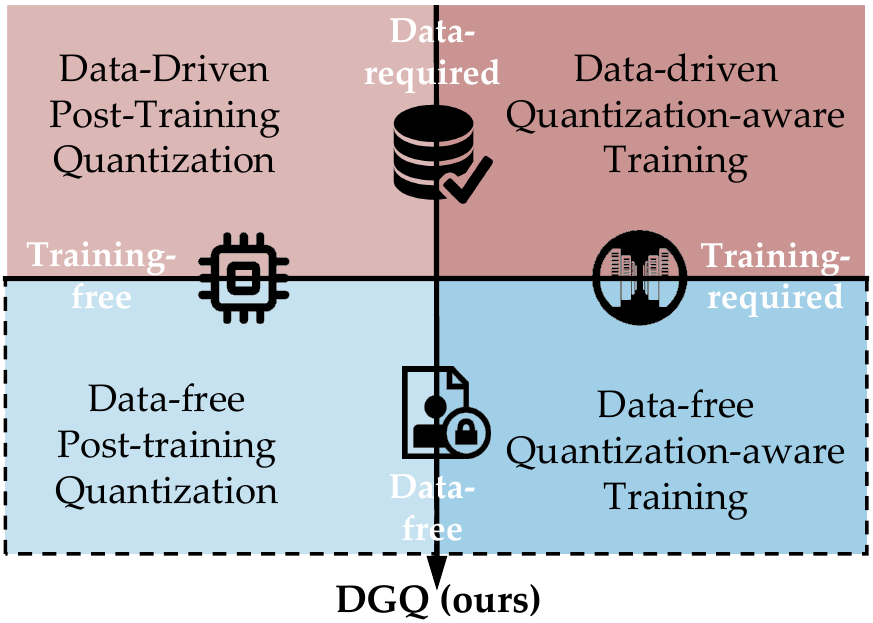}
\caption{The existing quantization approaches can be quartered by data-required or not and train/finetune-required or not. Our generic Diverse Sample Generation (DSG) method is proposed for data-free quantization approaches, including data-free post-training quantization and quantization-aware training.}
%\vspace{-0.1in}
\label{fig:categories}
\end{figure}

\begin{figure}[t]
%\vspace{-0.15in}
\centering
%\subfigcapskip=-5pt
\subfigure[]{
\includegraphics[width=0.48\linewidth]{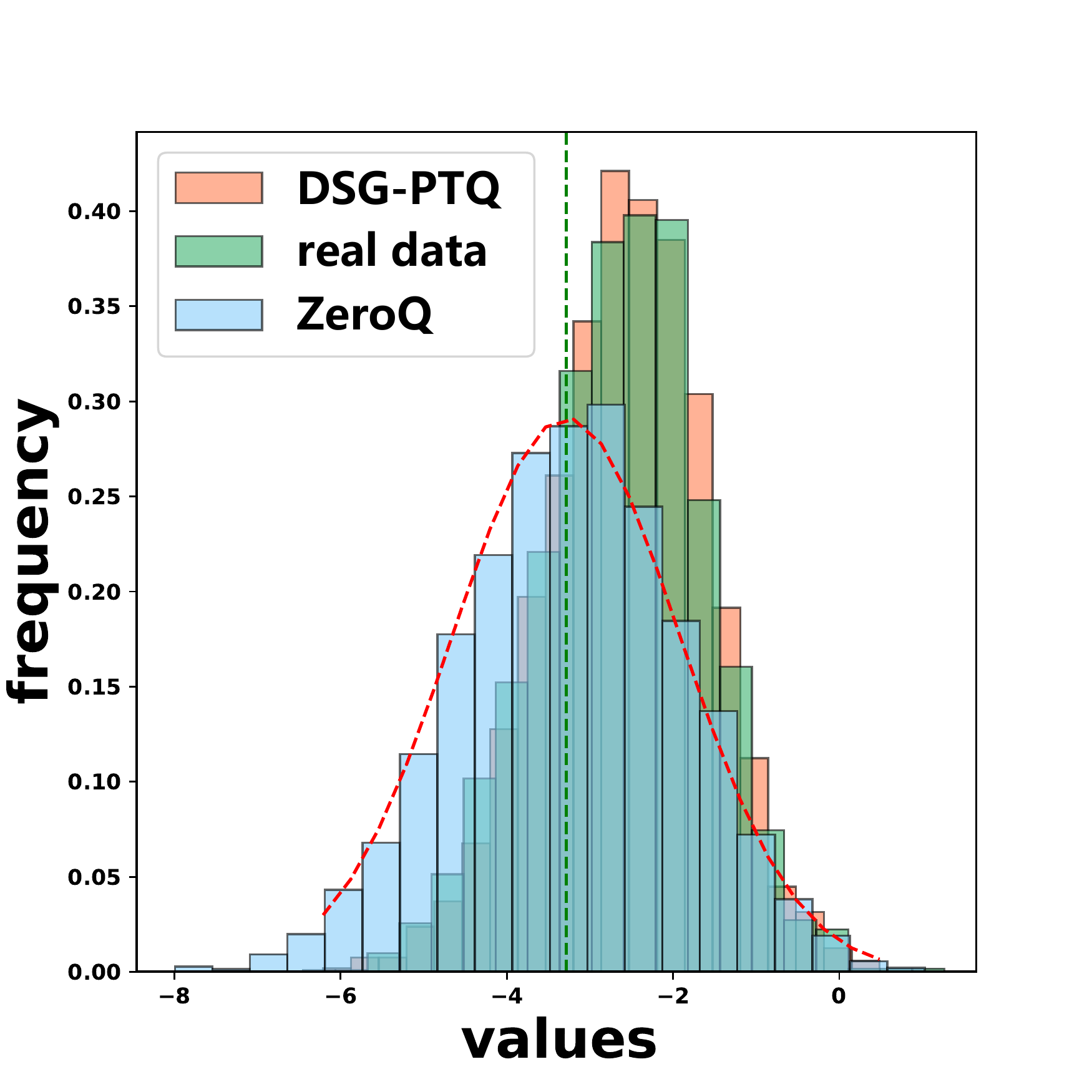}
\label{fig:distribution_a}
}
\hspace{-0.4cm}
\subfigure[]{
\includegraphics[width=0.48\linewidth]{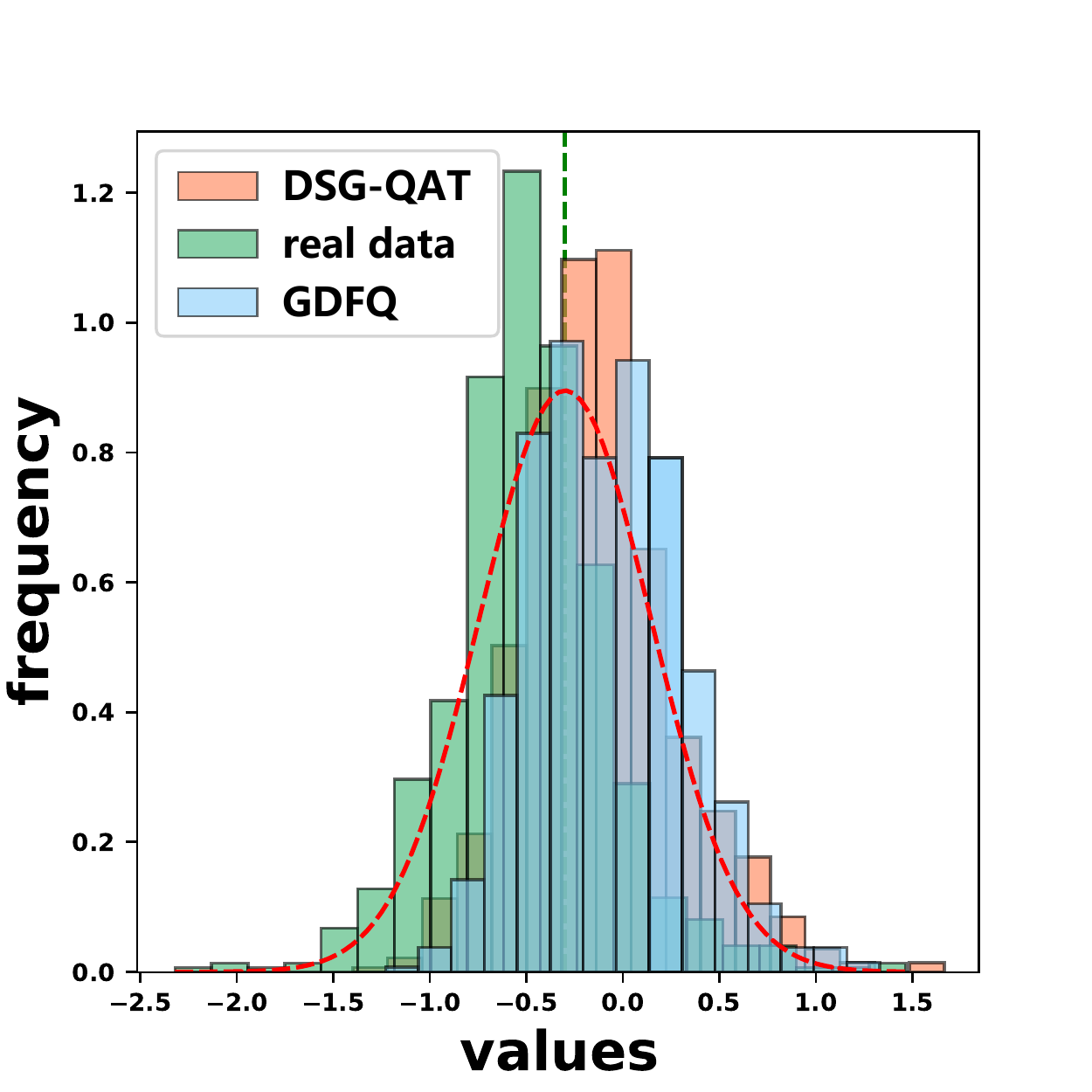}
\label{fig:distribution_c}
}\vskip -0.05in
    
\subfigure[]{
\includegraphics[width=0.48\linewidth]{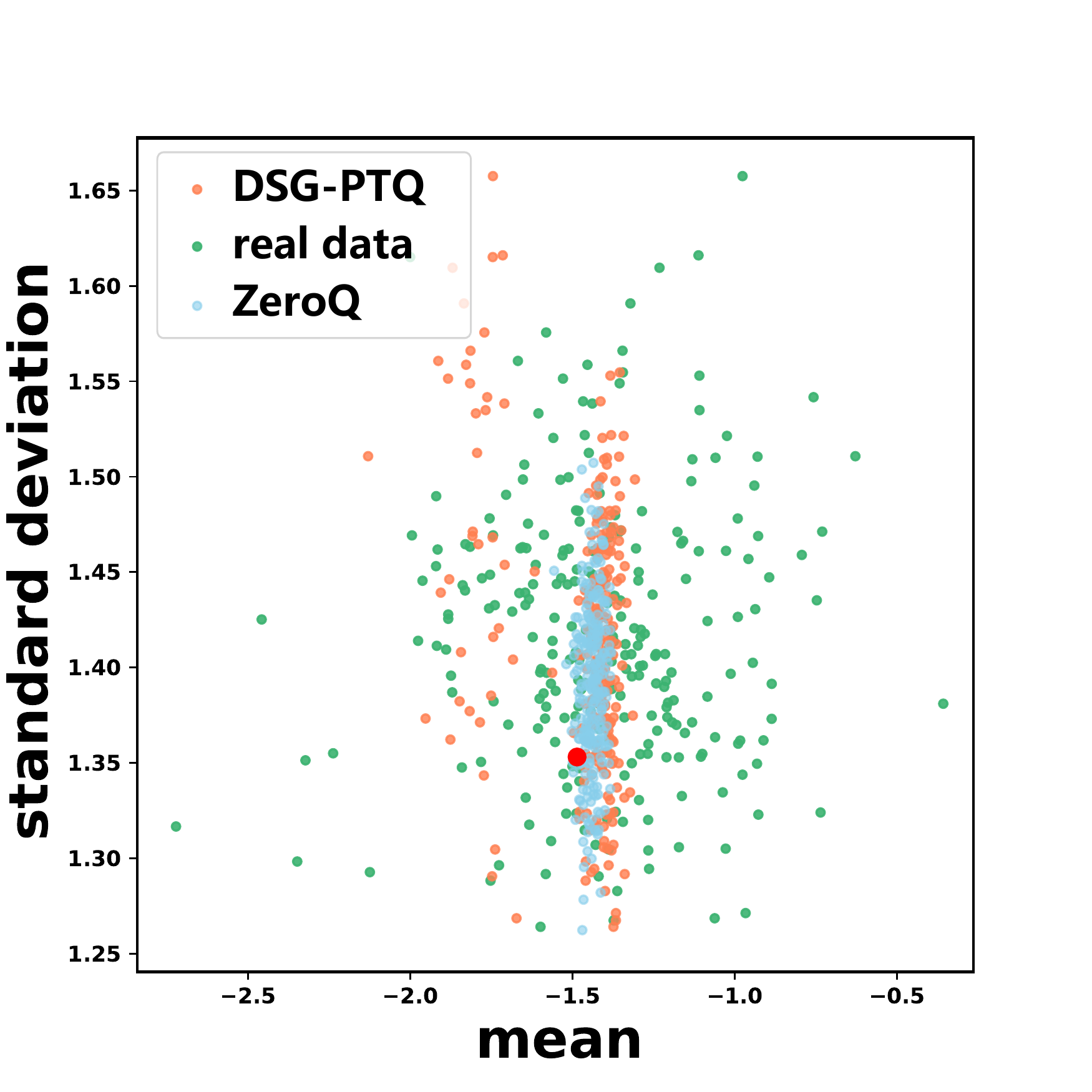}
\label{fig:distribution_b}
}
\hspace{-0.5cm}
\subfigure[]{
\includegraphics[width=0.48\linewidth]{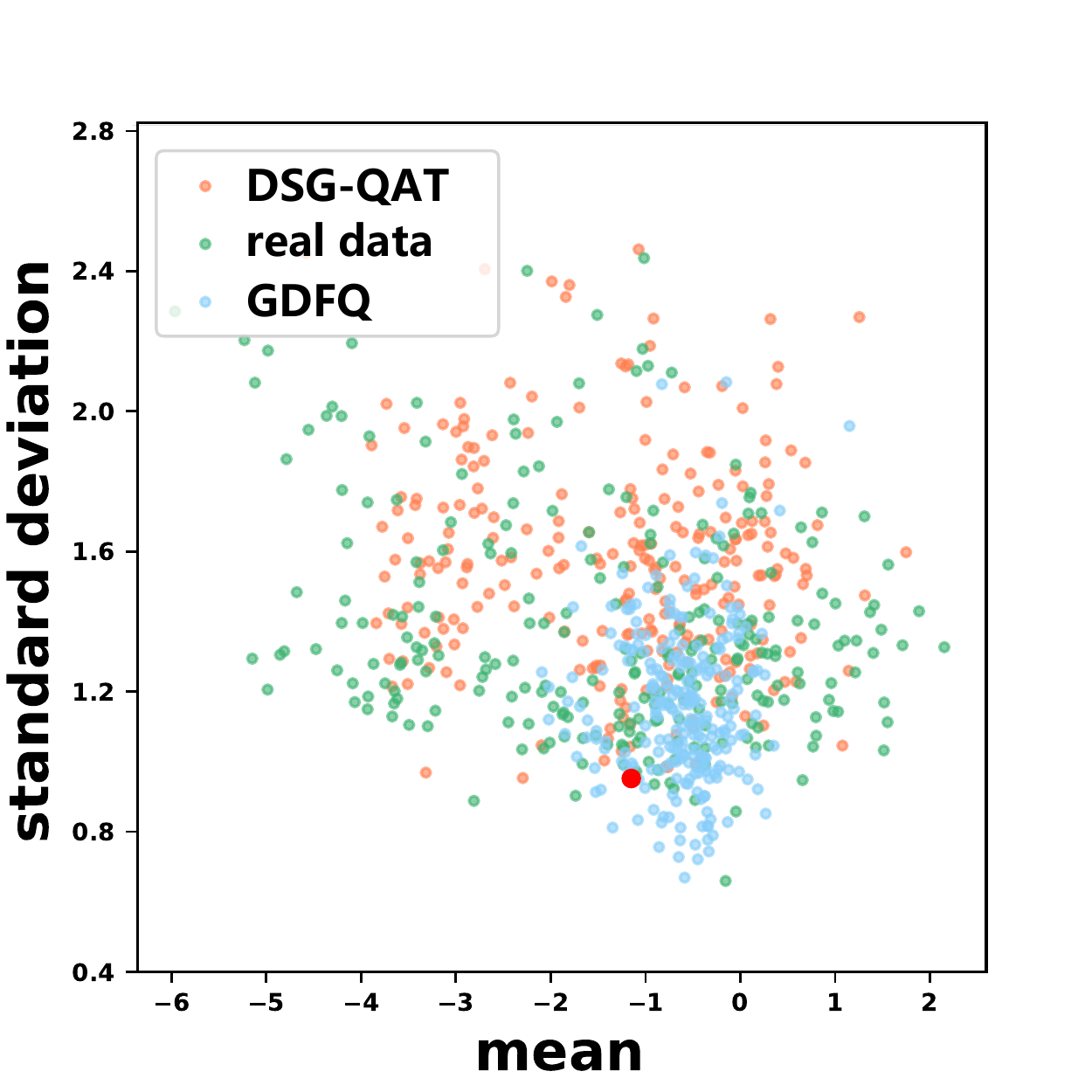}
\label{fig:distribution_d}
}\vskip -0.05in

%\hspace{-0.5cm}
\subfigure[]{
\includegraphics[width=0.41\linewidth]{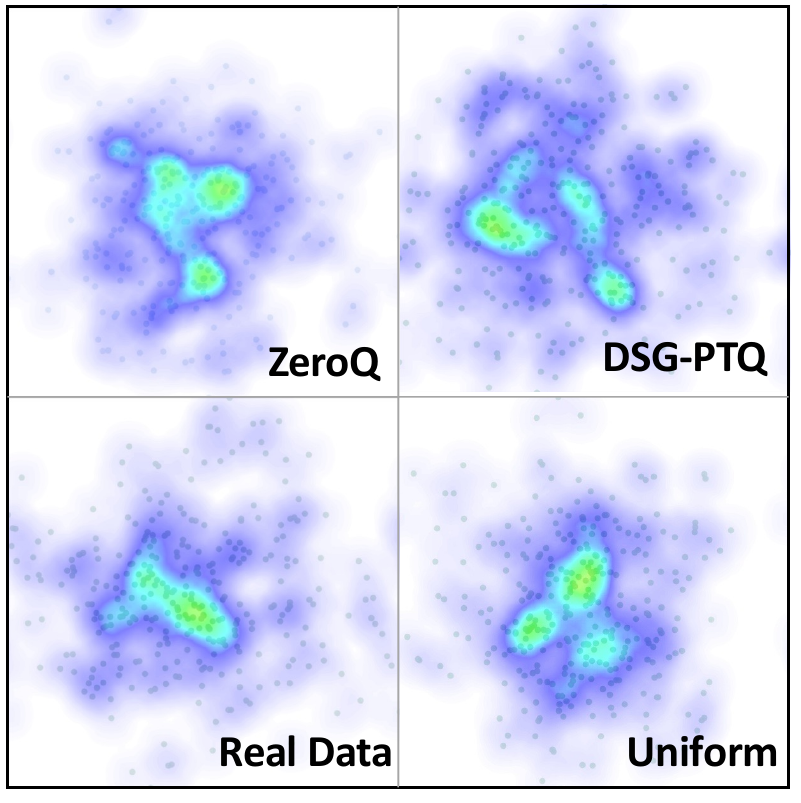}
\label{fig:space_e}
}
\hspace{0.2cm}
\subfigure[]{
\includegraphics[width=0.41\linewidth]{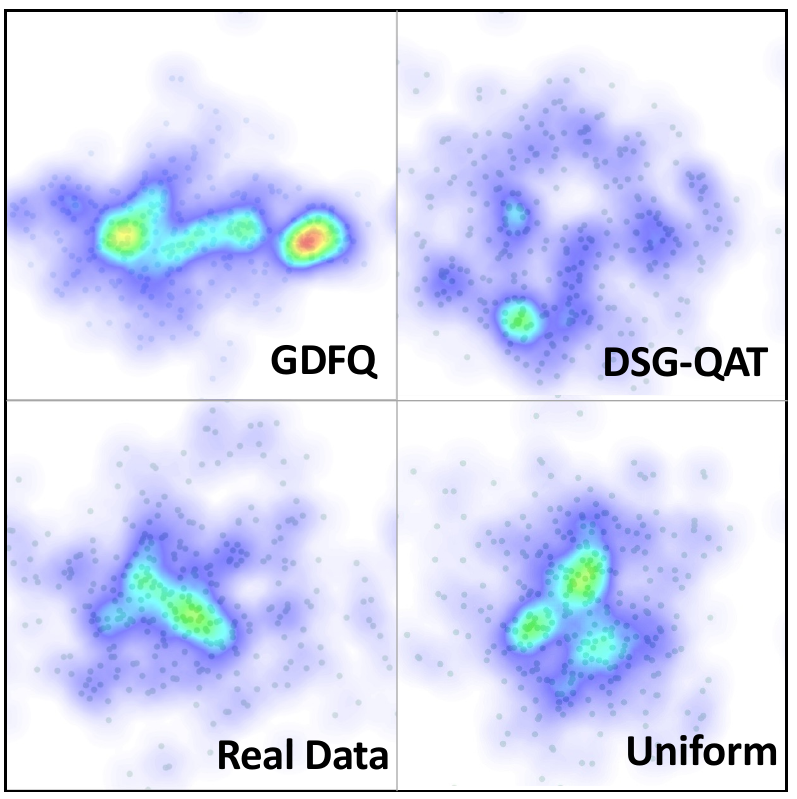}
\label{fig:space_f}
}
\vspace{-0.1in}
\caption{Comparison between real data (random sampling from ImageNet) and data synthesized by generative data-free quantization methods (DSG and ZeroQ for PTQ in the first column, DSG and GDFQ for QAT in the second column) with 256 samples of each. (a) and (b) show the activation distribution of one channel in ResNet18 for PTQ and QAT, respectively. 
(c) and (d) are charts of the mean and standard deviation for one channel of synthetic data. 
(e) and (f) are the visualization of the synthetic data after dimensionality reduction with PCA~\cite{pca}.
}
%\vspace{-0.1in}
\label{fig:distribution}
\end{figure}

Most quantization approaches are designed for the data-driven scenario, and thus require real data in their quantization process.
However, real data is not always accessible for privacy or security concerns (\textit{e.g.}, medical and user data).
Fortunately, data-free quantization is proposed to quantize deep neural networks without accessing real data.
As shown in Fig.~\ref{fig:categories}, the data-free quantization approaches can also be classified as data-free PTQ and data-free QAT according to whether the model training is required.
Among existing data-free quantization methods, generative methods calibrate or train networks using the "optimal synthetic data", which with the distribution best matches the Batch Normalization (BN) statistics of the original full-precision neural network.
Generative data-free PTQ methods apply BN statistics loss to directly update the generated data and then progress a computation-saving calibration process~\cite{Nagel_2019_ICCV,cai2020zeroq,Zhang_2021_CVPR}, while QAT ones generally introduce a separate generator for synthesizing data towards higher accuracy, and train it and the quantized network jointly~\cite{choi2020datafree,xu2020generative,he2021generative,yuang2021zaq}.

However, among existing quantization approaches, the data-driven methods always achieve significantly higher accuracy compared with data-free ones, whether for PTQ or QAT.
One intuitive reason is that the difference between real data and synthetic data leads to the gap in the performance between data-driven and data-free quantization approaches.
As a result, many approaches attempt to improve on synthetic data so that the performance of generative data-free quantization methods continues to approach data-driven quantization.
In this work, we first give a theoretical analysis from a diversity view for generative data-free quantization. 
Our analysis shows that synthetic samples should be distributed as dispersed in the input domain while satisfying given constraints. This is often described as the diversification of synthetic samples and would facilitate the optimization of the quantization to cover its potential input domain as possible.
Compared with real training and testing datasets that usually approximately satisfy the Independent and Identically Distributed (IID) assumption, synthetic data generated according to certain constraints in generative methods are difficult to meet the diversity requirement.

Corresponding to our theoretical analysis, we also experimentally reveals that severe homogenization exists in the data generation processes of existing typical data-free generative quantization methods (ZeroQ and GDFQ methods for PTQ and QAT, respectively).

\noindent \textbf{First}, the synthetic data is constrained to match the BN statistics, and thereby the feature distribution might overfit the BN statistics in each layer when the data forward propagate in models.
As shown in Fig.~\ref{fig:distribution_a} and Fig.~\ref{fig:distribution_c}, the distribution of samples synthesized by existing PTQ and QAT methods almost strictly follows the Gaussian distribution with corresponding BN mean and variance, while the distribution of real data has an obvious offset and enjoys a more diverse distribution.

\noindent We consider the first phenomenon as the \textit{\textbf{distribution-level homogenization}}.

\noindent \textbf{Second}, all samples of synthetic data are updated by one specific objective function in typical generative data-free quantization methods, all synthetic samples are applied to the same constraint and directly sum the loss term of each layer. In Fig.~\ref{fig:distribution_b} and Fig.~\ref{fig:distribution_d}, the distribution statistics of real data are dispersed while the data generated by existing approaches are centralized.

\noindent \textbf{Third}, for data-free QAT methods, synthetic samples are synthesized by a generator network.
In the process of learning the distributions of synthetic data for quantization, both the generator and the quantized network may converge to a trivial solution where the former learns to produce few modes exclusively, which is referred to by mode collapse and causes the synthetic data to clustered in sample space~\cite{elfeki2019gdpp}.
%The Fig.~\ref{fig:space_e} and Fig.~\ref{fig:space_f} are the visualization of the synthetic data after dimensionality reduction with Principal Component Analysis (PCA)~\cite{pca} in the sample space. 
As Fig.~\ref{fig:space_e} and Fig.~\ref{fig:space_f} show, the synthetic data is aggregated while the real data is scattered. 

\noindent The second and third phenomena are considered to be \textit{\textbf{sample-level homogenization}} from the statistical and spatial perspectives.

\noindent In a word, the distribution-level homogenization means that the overall distribution of synthetic data is strictly restricted to the specific distribution with BN statistics, while the sample-level homogenization means the little differences among samples from the statistical and spatial perspectives.
Only mitigating homogenization at one certain level cannot ensure the diversity of data at the other level.

To alleviate the accuracy degeneration of the quantized neural network caused by the homogenization of synthetic data, this paper presents a generic data generation scheme, \textbf{D}iverse \textbf{S}ample \textbf{G}eneration (DSG), for generative data-free quantization to enhance the diversity of the synthetic data. 
The proposed DSG scheme mainly relies on three novel techniques: 
\textit{Slack Distribution Alignment} (SDA): relax the distribution constraint of synthetic data by slacking the feature statistics alignment in each BN layer.  
\textit{Layerwise Sample Enhancement} (LSE): strengthen the impact of statistics loss of the specific BN layer for its corresponding synthetic sample by applying a layerwise enhancement.
\textit{Sample Correlation inhibition} (SCI): weaken the correlation among the synthetic samples by applying determinantal point processes loss to the intermediate features in the generation process.
Among these techniques, SDA diversifies the synthetic data at the distribution level, while LSE and SCI diversify it at the sample level from the statistical and spatial perspective, respectively.
Considering the generality, DSG focuses on the improvement of the generation process while could be almost decoupled from specific quantization methods.
Therefore, the proposed techniques of our DSG can be flexibly applied to different PTQ and QAT quantization approaches effectively, improving the accuracy of quantized neural networks by diversifying synthetic data in the generation process.

Our scheme presents a novel perspective of data diversity for generative data-free quantization, and extensive experiments show that DSG significantly improves both data-free PTQ and QAT.
The DSG performs remarkably well across several mainstream neural architectures, including VGG16bn~\cite{VeryDeepConvolutional} (the VGG16 with BN for the dense layers), ResNet18/20/50~\cite{he2016deep}, SqueezeNext~\cite{gholami2018squeezenext}, InceptionV3~\cite{szegedy2016rethinking}, ShuffleNet~\cite{zhang2018shufflenet}, and MobileNetV2~\cite{sandler2018mobilenetv2}, and surpasses the existing data-free methods in a wide margin on the large-scale image classification task and achieves the state-of-the-art (SOTA) results.
The quantized networks quantized by our DSG even outperform networks quantized by the real data under various settings.
The performance of the quantized neural networks trained by our DSG scheme is pushed to that of their full-precision counterparts, meanwhile, the quantization process gets rid of the data dependence.

We summarize the main contributions of this paper as:
\begin{itemize}
\item {We first give a theoretical analysis for the utility of data diversity in generative data-free quantization. Our analysis shows that synthetic samples should be distributed as dispersed while satisfying given constraints, which would facilitate the optimization of quantization to cover its potential input domain as possible.}
\item {We experimentally show and analyze the homogenization of synthetic data in existing generative data-free quantization. Our study presents the homogenization at the distribution and sample levels and reveals that the sample-level homogenization is not limited to the statistical perspective, but also the spatial perspective.}
\item {We propose a novel generic DSG scheme for generative data-free quantization from a comprehensive perspective of data diversity, which effectively improves the data-free PTQ and QAT. DSG presents a novel Sample Correlation Inhibition (SCI), in conjunction with the Slack Distribution Alignment (SDA) and Layerwise Sample Enhancement (LSE) techniques to diversify synthetic data at distribution and sample level.}
\item {We conduct a detailed ablation study on the proposed DSG, which presents the effectiveness of the proposed techniques (SDA, LSE, and SCI) in generative data-free PTQ and QAT.
And the comprehensive evaluation of the DSG scheme shows that our DSG surpasses the existing SOTA methods by a wide margin on various neural architectures and bit-widths, which demonstrates that the diversity is an important property of high-quality synthetic data.}
%\item {We further thoroughly study the synthetic data of generative data-free quantization methods. Benefits from the enhanced diversity, the performance of quantized networks in various methods (calibration and data-driven quantization methods) integrating with DSG data are significantly improved, which presents that diversity is an important property of high-quality synthetic data.}
\end{itemize}

Note that our paper extends the preliminary conference paper~\cite{Zhang_2021_CVPR}.
This manuscript first gives a theoretical analysis that diversifying the synthetic samples is a crucial element for improving the data-free quantization process, and experimentally reveals that the sample-level homogenization is not limited to the statistical perspective, but also the spatial perspective.
This work proposed a generic DSG scheme for generative data-free quantization, including both PTQ and QAT, while the previous scheme proposed in the conference paper can be regarded as a special case for PTQ.
The DSG in this work further presents a novel SCI technique to tackle the sample-level homogenization from the spatial perspective, which can both significantly improves the performance of various generative data-free quantization approaches. 
We evaluate our DSG scheme on various neural architectures and compared it with more generative data-free quantization methods. 
Besides, we further add more analysis and discussion of synthetic data and evaluate the data with various methods. The results present that diversity is an important property of high-quality synthetic data in the generative data-free quantization.

\begin{figure*}[t]
\centering
\includegraphics[width=1\linewidth]{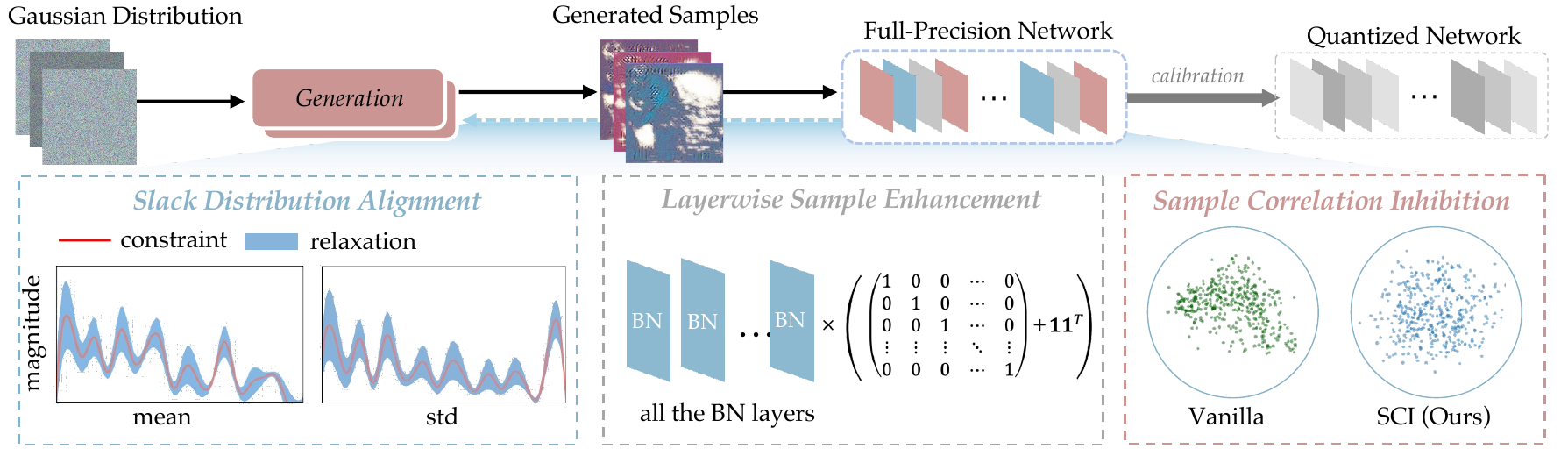}
\caption{The overview of the Diverse Sample Generation (DSG) scheme for data-free quantization. The proposed three techniques are applied to the generation process: Slack Distribution Alignment (SDA) relaxes BN statistics constraint in each layer; Layerwise Sample Enhancement (LSE) provides a specific loss term for each sample; Sample Correlation Inhibition (SCI) diversifies the synthetic data in sample space. The first one diversifies synthetic data at the distribution level, and the latter two at the sample level}
\label{fig:structure}
\end{figure*}

\section{Related Work}

\textbf{Data-driven Quantization.}
As a widely used compression technique, quantization earns lots of attention in recent years. To mitigate the severe accuracy drop resulted from quantization, many studies, such as \cite{gupta2015deep, jacob2018quantization,9454278,9072484}, utilize quantization-aware training methods to regain the accuracy. Generally, they can always provide performance improvements. However, the training process is quite cumbersome for huge computational and time costs. Another critical deficiency is that the datasets are not always ready to use, especially for privacy and security concerns. Without any finetuning or training process, research of post-training quantization has been conducted to improve the trade-off between accuracy and computational cost. \cite{banner2019posttraining} introduces two closed-form analytical solutions targeting clipping approximation and per-channel bit allocation combined with bias-correction to push the limit of weights and activations quantization to 4-bit. \cite{choukroun2019low} analyzes the mean squared error of quantization and minimizing it via the OMSE method. Unlike previous work which aims to find an appropriate clipping range to deal with outliers in the bell-shaped distribution, \cite{zhao2019improving} proposes outlier channel splitting while remains a functionally identical network. Not restricted to the round-to-nearest method, \cite{nagel2020down} exploits a more flexible weight-rounding mechanism called AdaRound, which optimizes task loss due to quantization w.r.t. the preactivation. However, these methods still require access to the limited data for a more resilient and better-quantized model.

\textbf{Data-free Quantization.}
Several recent studies have explored quantization methods without training and even calibration datasets. \cite{Nagel_2019_ICCV} uses weight equalization to acquire similar weight ranges among channels and bias correction for activations. However, quantizing models to run in 6-bit or lower bit-width is a non-trivial task. 
Works like \cite{haroush2020knowledge},  \cite{cai2020zeroq} and \cite{he2021generative} leverage BN statistics to reconstruct "realistic" data. 
Both \cite{haroush2020knowledge} and \cite{xu2020generative} utilizes knowledge distillation scheme to finetune the model, while it might cost a little more time. \cite{haroush2020knowledge} proposes an inception scheme to generate data, while \cite{he2021generative} utilizes a generative modeling to deal with the distribution of BN statistics. 
Besides, \cite{cai2020zeroq} can work on multi-bit quantization utilizing the Pareto frontier. 
\cite{choi2020datafree} and \cite{yuang2021zaq} focuses on adversarial approaches. \cite{choi2020datafree} proposes a data-free adversarial knowledge distillation method, which minimizes the maximum distance between the outputs of the full-precision teacher model and the quantized student model for any adversarial samples from a generator. 
\cite{yuang2021zaq} instead not only considers the output of model layer but also maximizes the discrepancy of feature map in intermediate inter-channel between teacher and student model by the proposed Channel Relation Map. 
Although the above methods use distinct tricks to synthesize data and improve model accuracy, they all have similar limitations and we take  \cite{cai2020zeroq} as an example to illustrate in Section~\ref{sec:homogenization} specifically. 

\section{Methods}

In this section, we first revisit the previous works of generative data-free quantization. Then, we reveal the data homogenization phenomenon in existing methods and improve the generation process to diversify the synthetic samples.

\subsection{Preliminaries}
\label{sec:preliminaries}

Existing data-driven quantization methods require original training and testing datasets to calibrate or finetune the quantized network for higher performance. But the real data is not accessible in many practical applications for privacy or security concerns, and thereby general quantization schemes cannot be applied directly to these scenarios. The data-free quantization is proposed to quantize the network without access to real data.

%\subsubsection{Batch Normalization Statistics Loss} 

Since the real data is inaccessible, the knowledge of the real data in the pre-trained network should be fully exploited in data-free quantization.
%The $\mathcal{L}_\textrm{PTQ}$ and $\mathcal{L}_\textrm{QAT}$ in Eq.~(\ref{eq:dfptq}) and Eq.~(\ref{eq:dfqat}) aims to constrain and update the synthetic data (or generator) through exploiting the knowledge in the full-precision network $\textrm{M}$, and then use the data to calibrate or finetune the quantized network.
The statistics (mean and standard deviation) of BN layers in full-precision models fit the original real dataset in the training process. Therefore, most data-free quantization schemes construct the synthetic data by using BN statistics loss to utilize the information in the BN layers.
The following optimization objective enables the distribution of synthetic data $\mathbf{x}^s$ to best match the BN statistics, including the mean and standard deviation:
\begin{equation}
\label{eq:BNS}
\min\limits_{\mathbf{x}^s}\mathcal{L}_\textrm{BN}=
\sum\limits_{i=1}^N
\left\|\boldsymbol{\tilde{\mu}}_{i}^s-\boldsymbol{\mu}_{i}\right\|_{2}^{2}+\left\|\boldsymbol{\tilde{\sigma}}_{i}^s-\boldsymbol{\sigma}_{i}\right\|_{2}^{2},
\end{equation}
where $\boldsymbol{\tilde{\mu}}_{i}^s$ and $\boldsymbol{\tilde{\sigma}}_{i}^s$ are the mean and standard deviation of feature distribution of synthetic data $\mathbf{x}^s$ at $i$-th BN layer, while $\boldsymbol{\mu}_{i}$ and $\boldsymbol{\sigma}_{i}$ are that of pre-trained full-precision model $\textrm{M}$. By minimizing the BN statistics loss $\mathcal{L}_\textrm{BN}$, the synthetic data or the generator learns a distribution of input data to best match BN statistics of each layer.

%\subsubsection{Generative Data-free Quantization}

%Given a neural network $\textrm{M}$, generative data-free quantization aims to generate synthetic data $\mathbf{x}^s$ and meanwhile quantize the network to low-bit. 
%As mentioned in Section~\ref{sec:intro}, data-free PTQ and QAT are applied in scenarios where resources and time are limited and sufficient, respectively. 
Among most of the mainstream generative data-free quantization methods, the PTQ methods directly update the data from random vectors during the generation process and then use the synthetic data to calibrate the quantized neural network.
The typical optimization objective in generation process of these methods can be expressed as
\begin{equation}
\label{eq:dfptq}
\min\limits_{\mathbf{x}^s} \mathcal{L}_\textrm{PTQ}\left(\mathbf{x}^s, \textrm{M}\right),
\end{equation}
%where $\mathbf{x}^s$ is optimized directly by reducing the loss function. 
where $\textrm{M}$ denotes the full-precision model, and the loss function $\mathcal{L}_\textrm{PTQ}$ is usually the BN statistic loss $\mathcal{L}_\textrm{BN}$.
The synthetic data $\mathbf{x}^s$ is then used to calibrate the quantized model, a process that is usually completely separate from the previous generation process.

While the generative data-free QAT methods train a separate generator model to generate the synthetic data and use them to train the quantized neural network.
%While for the tolerance of computational complexity in generative data-free QAT, we usually use an extra generator network $\textrm{G}$ to synthetic data.
The typical generation process of these methods is expressed as follows:
\begin{equation}
\label{eq:gen}
\min\limits_{\textrm{G}} \mathcal{L}_\textrm{QAT}\left(\textrm{G}(\mathbf{x}^{s*}), \textrm{M}\right),\quad 
\mathbf{x}^s=\textrm{G}(\mathbf{x}^{s*}),
\end{equation}
where $\mathbf{x}^{s*}$ is an initial tensor drawn from standard Gaussian distribution $\mathcal{N}(0, 1)$, the loss function $\mathcal{L}_\textrm{QAT}$ usually consists of the BN statistical loss and the cross-entropy function of the generator $\textrm{G}$, \textit{etc.}, and $\mathbf{x}^s$ and $y$ are synthetic data and the corresponding label, respectively.
And the quantized neural network $\textrm{Q}$ is trained simultaneously using synthetic data:
\begin{equation}
\label{eq:dfqat}
\min\limits_{\textrm{Q}} \mathbb{E}_{\mathbf{x}^s, y}\left[\mathcal{C}\left(\textrm{Q}, (\mathbf{x}^s; y)\right)\right],
\end{equation}
where $\mathcal{C}(\cdot,\cdot)$ is a loss function for quantized neural network $\textrm{Q}$, such as cross-entropy loss and mean squared error, and $y$ is the corresponding labels of synthetic data $\mathbf{x}^s$.

\subsection{Diversity of Synthetic Data for Quantization}
\label{sec:why}

In generative data-free quantization and other related generative works, diverse synthetic data is always more welcome intuitively. But there is little relevant theoretical analysis to support this intuition, especially in the context of the data synthesis with strict constraints. It is important to understand the theoretical motivation behind this intuition. Therefore, we give a theoretical analysis to show that in various generative data-free quantization approaches, the diversity of synthetic data facilitates the quantization process.

\begin{myDef}
\label{def:dfq}
For a well-trained full-precision model $\textrm{M}$, the optimization objective of the whole generative data-free quantization process can be abstracted as
$$\min\limits_{\mathbf{x}^s, \rm{Q}} \mathcal{L}_1\left(\mathbf{x}^s, \rm{M}\right) + \mathcal{L}_2\left(\rm{Q}, \mathbf{x}^s\right),\quad \mathbf{x}^s\subseteq \mathcal{X},$$ 
where $\mathbf{x}^s$ is the synthetic data, $\rm{M}$ and $\rm{Q}$ are the full-precision and quantized networks, respectively, and $\mathcal{X}$ denotes the input domain of there networks. $\mathcal{L}_1$ and $\mathcal{L}_2$ are constraints for the data generation and network quantization.
\end{myDef}

According to the discussion in Section~\ref{sec:preliminaries}, we can summarize that, in generative data-free quantization methods, the key elements are synthetic data, full-precision model, and corresponding quantized model. The full-precision model is well-trained and fixed, while the quantized model and synthetic data are optimizable. Therefore, we abstract and define the general generative data-free quantization process as Definition~\ref{def:dfq}. In this summarized definition, a crucial implicit premise is that the synthetic data $\mathbf{x}^s$ should reflect the input domain $\mathcal{X}$ in the process. The optimization process of the quantized model $\textrm{Q}$ will perform well only if this premise is well satisfied.
However, the input domain $\mathcal{X}$ is not explicit in the process, and the real data, which is usually regarded as the ideal set reflecting the input domain, even cannot be obtained.

%It is generally acknowledged that inﬁnite amounts of labeled samples are not available. We collect finite examples as sampling from the immense open world. Empirically, quantity and diversity are highly required for a large-scale dataset to cover the scope. And the principles are also the same when it comes to data synthesizing.

\begin{myLem}
\label{lem:app1}
For any input domains $\mathcal{X}$ that includes multiples classes (at least 2) of samples, it can be modeled as several independent high-density $\{\mathcal{R}_{H1}, \cdots, \mathcal{R}_{Hh}\}$ and low-density $\{\mathcal{R}_{L1}, \cdots, \mathcal{R}_{Ll}\}$ sub-regions divided by possible decision surfaces, where $h\geq 1$ and $l\geq 0$.
\end{myLem}

Therefore, for generative data-free quantization, it is necessary to study how to allow synthetic samples to reflect the properties of an unknown input domain well. Fortunately, there are two classical assumptions about the input domain concluding multiple classes of data in deep learning, and they help reveal the properties of the ideal set of synthetic samples.
(1) \textit{Low-density assumption}. The assumption indicates that the boundary of class regions should be more likely to avoid the high-density centers and pass through the low-density areas. 
(2) \textit{Smoothness assumption}. The assumption states that for two points in the input domain that are close by in the input space, the corresponding labels should be the same.
As the Lemma~\ref{lem:app1}, these assumptions allow us to model the potential input domain as several independent sub-regions, and we provide the proof of this modeling in Appendix~\ref{app:lem1}.
%证明（反证）：如果不存在低密度区域（存在且只存在1个高密度区域），根据平滑假设，决策面不能穿过任何地方，整个输入域只有一类；除此以外，决策面可以随意穿越输入区域，使其分割成多个整数个高密度或低密度区域，决策面划分类别>=标签类别。

When we model a potential input domain $\mathcal{X}$ in generative data-free quantization in this way, consider a discrete set of synthetic samples $\mathbf{x}^s=\{x^s_0, x^ s_1, \cdots, x^s_N\}$, we can intuitively get that a necessary condition for the sample set to well reflect the input domain is to well reflect all sub-regions of the input domain, because each sub-regions are divided as independent and are not representative for others. 
And for the sample generation, the specific shape and characteristics of the potential input domain are uncertain, so the discussion of the input domain should face all possible input domains rather than a specific one.

\begin{myTheo}
\label{th:dfq}
%A given input domain $\mathcal{X}$ of scale $V$ consists of several sub-regions $\{R_0, R_1, \cdots\}$ of scales $\{V_0, V_1, \cdots\}$.
Given a set of all possible input domains $\mathbf{X}=\{\mathcal{X}_0, \mathcal{X}_1, \cdots \}$, whose $i$-th element can be denoted as $\mathcal{X}_i$ with scale $V^i$ and consists of several sub-regions $\{\mathcal{R}^i_1, \cdots, \mathcal{R}^i_{K^i}\}$ with scales $\{V^i_1, \cdots, V^i_{K^i}\}$, and the number $K^i\geq 2$ is unknown yet limited. Consider a sample set $\mathbf{x}^s=\{x_0^s,\cdots , x_N^s\} \subset \mathcal{X}^*$, where $\mathcal{X}^*=\mathbb{E}(\mathbf{X})$ denotes the potential input domain and the differences inside each sub-region of $\mathcal{X}$ is neglected. When the set $\mathbf{x}^s$ satisfies that for $\forall x_i^s\in\mathbf{x}^s$, $p (x_i^s\in \mathcal{R}^*_j)=\frac{V^*_j}{V^*}$, the information reflecting from all possible input domains $\mathbf{X}$ by the sample set $\mathbf{x}^s$ will be the maximized in mathematical expectation, where $V^*=\sum_{k=0}^{K^*}V^*_k$. 
%换成密度*体积表述
%因为内部距离被忽略，根据平滑假设密度无意义
\end{myTheo}

Theorem~\ref{th:dfq} shows that the optimized strategy to make synthetic samples adequately reflect the unknown input domain in expectation is to correlate the distribution of samples with the spatial scales of the sub-regions, \textit{i.e.}, dispersing synthetic samples uniformly in the potential input domain instead of gathering collapse. We present the proof of the theorem in Appendix~\ref{app:th1}.
In practice, this means that synthetic samples should be distributed as dispersed in the domain as possible while satisfying given constraints, which is often described as the diversification of synthetic samples.

\subsection{Homogenization of Synthetic Data}
\label{sec:homogenization}

The generative data-free quantization well sidesteps the issue of lacking accessibility of real data, but we have to concede the huge performance drop compared with networks quantized with real data. 
As the theoretical analysis in Section~\ref{sec:why}, we show that the diversity of synthetic data is beneficial for generative data-free quantization, thus we attempt to find the performance bottleneck of existing methods from this aspect.
By experimental analysis for the typical generative data-free PTQ and QAT methods with the almost independent generation process (ZeroQ and GDFQ, respectively), we discover the homogenization phenomena from the existing image synthesizing methods in data-free quantization, which degrade the fidelity and quality of synthetic data. We conclude the homogenization of synthetic data into the distribution level and the sample level to elaborate on them.

\subsubsection{Distribution-level Homogenization} 

Since the performance of existing data-driven quantization significantly surpasses that of data-free quantization in usual, the original training data (real data) is generally considered as the upper limit of the performance of synthetic data. 
So as Eq.~(\ref{eq:BNS}) shows, methods constrained by $\mathcal{L}_\textrm{BN}$ are dedicated to generating samples by confining the feature statistics of synthetic data around BN statistics of the pre-trained model. 
However, the data from the real world is always more diverse, while that generated by fitting BN statistics suffers constrained features as shown in Fig.~\ref{fig:distribution_a} and Fig.~\ref{fig:distribution_c}. The activation distribution of data generated in ZeroQ and GDFQ well fit the normal distribution of BN statistics while that of real data deviates from BN statistics.
When we calculate the Wasserstein distance between the Gaussian distribution with BN statistics and the feature distribution of synthetic data, the index values of the existing methods are significantly smaller than the real data (real data 0.120 vs ZeroQ 0.040 in PTQ, and real data 0.134 vs GDFQ 0.116 in QAT).
We consider it as distribution-level homogenization causing by the strict constraint applied to the synthesizing process makes the synthetic data overfit BN statistics in each layer. 

\subsubsection{Sample-level Homogenization} 

Besides homogenization at the distribution level, synthetic data in existing generative data-free quantization also suffer the homogenization at the sample level.

\textbf{From the statistical perspective}, in existing generative data-free methods, such as ZeroQ and GDFQ, all samples are constrained by the identical form of BN statistics loss. So there might be little difference among statistics of synthetic samples in one batch, while samples from the real world are diverse. As shown in Fig.~\ref{fig:distribution_b} and Fig.~\ref{fig:distribution_d}, the points of the mean and standard deviation of ZeroQ and GDFQ data are significantly overlapping and all points gather together closely near BN statistics. However, statistics of real data and our DSG data have larger variance and are more dispersed.
Taking PTQ as an example, the variance of the statistics of the real sample is 0.029, that much larger than the variance of data generated by ZeroQ, which is 0.009.

\textbf{From the spatial perspective}, the synthetic samples generated in different ways have different behaviors.
For existing data-free methods, all samples are synthesized by the same generation strategy. The samples generated by GDFQ are concentrated in spatial perspective, as shown in Fig.~\ref{fig:space_f}, while samples of our DSG and real data are more scattered.
The heatmap shows the degree of aggregation in the synthetic data sample space. The densest cluster of GDFQ samples appears red in the heatmap and the density index is up to 2052 near the cluster center, while the highest density index of the real sample is only 1681.
For PTQ, the index is 1902, and its homogenization phenomenon is not significant as QAT but is more severe than the real data.
%However, for PTQ methods, there are no significant differences among the four types of samples as shown in Fig.~\ref{fig:space_e}. It seems that the sample-level homogenization from the spatial perspective is a particular property in data-free generative QAT methods.
We regard these as sample-level homogenization from the statistical and spatial perspectives, which leads to a significant performance drop for the networks quantized by synthetic data.

In a nutshell, for the generation process in data-free quantization, the homogenization exists in two levels, \textit{i.e.}, distribution level and sample level. Thus the models quantized by these homogenized data can hardly achieve high accuracy. In this work, we devote diversify the synthetic data to improve the generative data-free quantization. We propose a generic data-free quantization scheme to tackle the homogenization problems by diversifying synthetic data from different perspectives, which we call Diverse Sample Generation (DSG) scheme. The synthetic data generated by our scheme enables the quantized neural network to achieve high accuracy.

\subsection{Diverse Sample Generation}

We present a generic Diverse Sample Generation (DSG) scheme for both generative data-free PTQ and QAT approaches (Fig.~\ref{fig:structure}), including three novel techniques to alleviate homogenization: Slack Distribution Alignment in Section~\ref{sec:slack} is to relax the statistical constraint and thus mitigate distribution-level homogenization, while Layerwise Sample Enhancement in Section~\ref{sec:layerwise} and Sample Correlation Inhibition in Section~\ref{sec:correlation} diversify the synthetic data at the sample level from distribution and spatial perspectives, respectively.

\subsubsection{Slack Distribution Alignment}
\label{sec:slack}

To deal with the distribution-level homogenization, we propose Slack Distribution Alignment (SDA). Specifically, we apply the relaxation constants during the generation process for matching original BN statistics, which can be intuitively regarded as margins for means and standard deviation. With the relaxation constants, the distribution statistics of synthetic samples do not need to fit the BN statistics strictly, and thus the data can be more diverse in distribution. The loss term of SDA is $\mathbf{l}_\textrm{SDA}=[{l}_{\textrm{SDA}_1}, {l}_{\textrm{SDA}_2}, \dots, {l}_{\textrm{SDA}_N}]^T$, where $N$ is the number of BN layers in the quantized neural network, and the ${l}_{\textrm{SDA}_i}$ for $i$-th BN layer is expressed as follow:
\begin{equation}
\label{eq:SDA}
\begin{aligned}
l_{\textrm{SDA}_i}=&
\left\|\max\left(|\boldsymbol{\tilde{\mu}}_{i}^s-\boldsymbol{\mu}_{i}|-\delta_i, 0\right)\right\|_{2}^{2}\\
+&\left\|\max\left(|\boldsymbol{\tilde{\sigma}}_{i}^s-\boldsymbol{\sigma}_{i}|-\gamma_i, 0\right)\right\|_{2}^{2},
\end{aligned}
\end{equation}
where $\delta_i$ and $\gamma_i$ denote the relaxation constants for the mean and standard deviation statistics of features at the $i$-th BN layers, respectively.
After introducing relaxation to the constraint between statistics of synthetic data and BN layers, the mean and standard deviation of generated samples can fluctuate within a certain range. Benefiting from the relaxed constraints, the synthetic samples are more diverse and their feature distribution behaves closer to that of real data.
%as shown in Fig.~\ref{fig:distribution_a} and Fig.~\ref{fig:distribution_c}.

As the ideal data to quantize networks, our experimental evidence in Fig.~\ref{fig:distribution} shows that the feature statistics of real data still exist gaps compared with BN statistics. Thus, we attempt to approximate the gap between the statistics of real data and BN statistics of the original model as a reference to determine the value of $\delta_i$ and $\gamma_i$.
For real-world data, when the number of samples achieves a huge amount, the components and features of inputs approximate the Gaussian assumption. Based on the central limit theorem, we suppose the whole potential input set is generally consistent with a Gaussian distribution. 

Thus, we set the initial value of $\delta_i$ and $\gamma_i$ by a batch of samples that were randomly initialized from the Gaussian distribution: 
First, we input 1024 samples initialized by Gaussian distribution with $\mu=0$ and $\sigma=1$ into the full-precision network and then save the feature statistics at each BN layer, including the mean and standard deviation of the feature distribution. 
Then, we measure the margin between saved statistics and the corresponding BN statistics, then take percentiles of the margin to initialize $\delta_i$ and $\gamma_i$: 
\begin{equation}
\label{eq:mn}
\begin{aligned}
\delta_i=|\boldsymbol{\tilde\mu}^{0}_{i}-\boldsymbol{\mu}_{i}|_{\epsilon} \quad
\gamma_i=|\boldsymbol{\tilde\sigma}^{0}_{i}-\boldsymbol{\sigma}_{i}|_{\epsilon},
\end{aligned}
\end{equation}
where $\epsilon \in (0, 1]$ is a hyper-parameter determining the degree of relaxation, and $\boldsymbol{\tilde\mu}^{0}_{i}/\boldsymbol{\tilde\sigma}^{0}_{i}$ are mean/standard deviation of the feature of Gaussian initialized data $\mathbf{x}^{0}$ at the $i$-th BN layer.
$|\boldsymbol{\tilde\mu}^{0}_{i}-\boldsymbol{\mu}_{i}|_{\epsilon}$ and $|\boldsymbol{\tilde\sigma}^{0}_{i}-\boldsymbol{\sigma}_{i}|_{\epsilon}$ denote the $\epsilon$ percentile of $|\boldsymbol{\tilde\mu}^{0}_{i}-\boldsymbol{\mu}_{i}|$ and $|\boldsymbol{\tilde\sigma}^{0}_{i}-\boldsymbol{\sigma}_{i}|$, respectively. 
Intuitively, the larger the value of $\epsilon$, the more relaxing the constraints are in Eq.~(\ref{eq:SDA}). 
We empirically set the default value of $\epsilon$ as $0.9$ to deal with outliers.

\begin{figure*}[t]
\centering
\includegraphics[width=1\linewidth]{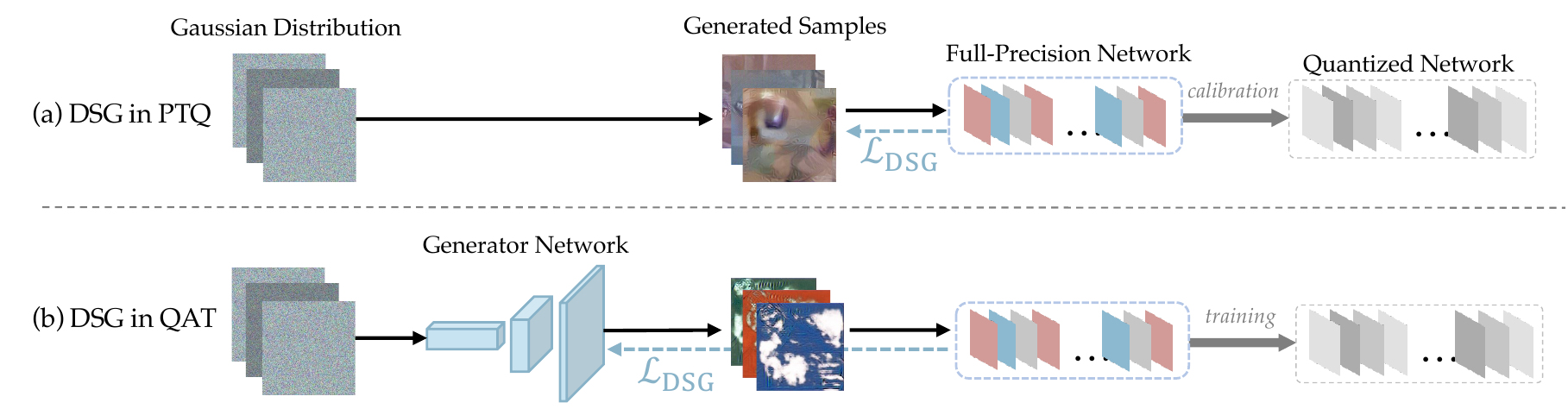}
\vspace{-0.2in}
\caption{The procedures of our DSG for specific quantization approaches. The case (a) is the DSG in PTQ and the case (b) is in QAT.}
\label{fig:specific-arch}
\end{figure*}

\subsubsection{Layerwise Sample Enhancement}
\label{sec:layerwise}

For the generation processes in existing generative data-free quantization approaches, all synthetic samples are always generated by the same initialization and optimization strategy. Specifically, the samples share the same objective function that directly sums the loss of each layer, \textit{i.e.}, all the BN layers receive the same degree of attention from each sample, and thus results in homogenization among the synthetic samples.
We propose Layerwise Sample Enhancement (LSE) to tackle sample-level homogenization. LSE enables each sample to focus on the statistics of different layers by enhancing the specific loss term of the sample in the optimization process so that the synthetic samples of our DSG in one batch could be diverse.

Given a well-trained full-precision network containing the $N$ number of BN layers, earlier generative data-free quantization methods treat each layer the same way without particular predilection or bias. But in LSE, we engineer $N$ different loss terms and apply each of them to the specific data sample.
Here, we suppose to generate $N$ samples in one batch, equaling the number of BN layers of the quantized network. The loss term of LSE for the $i$-th sample is as follows:
\begin{equation}
\label{eq:LSE}
{l}_{\textrm{LSE}_i}=\frac{1}{N}(\mathbf{1}+\textbf{h}_i)^T\mathbf{l},
\end{equation}
where $\mathbf{1}$ denotes a column vector of all ones, $\mathbf{h}_i$ is an ${N}$-dim one-hot column vector where the $1$ is in the $i$-th position, $\mathbf{l}=[l_1, l_2, ..., l_N]^T$ is the vector of the original loss terms for all BN layers.
The set of loss functions for all samples can be expressed as $\mathbf{X}_\textrm{LSE}\mathbf{l}=[{l}_{\textrm{LSE}_1}, {l}_{\textrm{LSE}_2}, ..., {l}_{\textrm{LSE}_N}]^T$, which is an $N$-dim column vector and its $i$-th element represents the loss function for $i$-th sample in the batch.
Thus, for the whole set of samples $\mathbf{x}^s$, the effect exerted by LSE can be formulated as
\begin{equation}
\label{eq:LSE_expand}
\mathbf{X}_\textrm{LSE} = \frac{1}{N}\left(\mathbf{1}\mathbf{1}^T+[\textbf{h}_1, \textbf{h}_2, ..., \textbf{h}_N]\right),
\end{equation}
where $\mathbf{X}_\textrm{LSE}$ can be seen as the enhancement matrix for loss terms. 
In this way, we focus on every BN layer of the original model respectively and generate corresponding samples for each layer. By such layerwise enhancement, we optimize the whole batch of synthetic samples simultaneously.

\subsubsection{Sample Correlation Inhibition}
\label{sec:correlation}

The above two proposed methods are based on the relaxation of generation constraints, however, the state of samples at the spatial level still cannot be perceived and adjusted directly during the generation process. Therefore, in addition to improving the loss function constraining the statistics of synthetic data, we construct the unsupervised Sample Correlation Inhibition (SCI) loss for the samples to sample-level diversify the synthetic data from the spatial perspective.
SCI applies the determinant point process loss to spatially disperse the intermediate features in the generator, thereby inhibiting the correlation among these synthetic samples.
%We define the determinant point process and related notions in Definition~\ref{def:dpp}.

%\begin{myDef}
%\label{def:dpp}
%A point process $\mathcal{P}$ on a ground set $\mathcal{V}$ is a probability measure on the power set $2^{G}$, where $G=|\mathcal{V}|$ is the size of the ground set.
%A point process $\mathcal{P}$ is called determinantal if, given a random subset $\mathbf{Y}$ drawn by $\mathcal{P}$, we have for every $\mathbf{S} \subseteq \mathbf{Y}$, $\mathcal{P}(\mathbf{S} \subseteq \mathbf{Y}) \propto \operatorname{det}\left(\mathbf{K}_{\mathbf{S}}\right)$ for some symmetric similarity kernel $\mathbf{K} \in \mathbb{R}^{G\times G}$, where $\mathbf{K}_{\mathbf{S}}$ is the similarity kernel of subset $\mathbf{S}$.
%$\mathbf{K}_{\mathbf{S}}$ denotes the submatrix of $\mathbf{K}$ indexed by ${\mathbf{S}}$, specifically, $\mathbf{K}_{\mathbf{S}} \equiv\left[\mathbf{K}_{i j}\right] ; i, j \in {\mathbf{S}}$.
%A large value of $\mathbf{K}_{i j}$ reduces the likelihood of both elements to appear together in a diverse subset.
%\end{myDef}

The determinant point process $\mathcal{P}$ can express negative interactions among samples by the similarity kernel $\mathbf{K}_{\mathbf{S}}$ of the elements in a set ${\mathbf{S}}$, where a large value of $\mathbf{K}_{i j} ; i, j \in {\mathbf{S}}$ reduces the likelihood of both elements to appear together in a diverse subset~\cite{borodin2009determinantal,kulesza2012determinantal}. 
The related works are widely used to generate or select diverse samples in some existing generation tasks for diversifying samples closer to the real data~\cite{elfeki2019gdpp,chen2018fast,li2017mmd,liu2021random}.
However, real data is inaccessible in generative data-free quantization, and thereby it cannot be used to construct the constraint of the diversity of synthetic data.

Therefore, we utilize random samples in our SCI instead of the real data, and the random data is applied as a lower bound of diversity to alleviate the homogenization of synthetic samples. 
Inspired by Theorem~\ref{th:dfq}, each dimension of these vectors is initialized by the uniform distribution $U[0, 1)$ independently and randomly. The vectors initialized in this way are for uniforming approximately on spatial, which are fixed during the quantization process and expected to compose a base set that covers the potential input domain uniformly and is de-homogenized.
In this way, the homogenization of synthetic data is limited to a level lower than its initial state.
%, \textit{i.e.}, and it will not bring additional homogenization in the generation process.

So far we are able to build the complete flow of our SCI.
The intermediate features correspond to the data $\mathbf{x}^s$ in the generation process is expressed as $\mathbf{f}=\{f_1, f_2, ..., f_N\}$, where $f_i$ is the feature of the corresponding synthetic sample.
And for the set of features $\mathbf{f}$, we construct a set of noise vectors $\mathbf{r}=\{r_1, r_2, ..., r_N\}$ of the same size. 
%Inspired by Theorem~\ref{th:dfq}, each dimension of these vectors is initialized by the uniform distribution $U[0, 1)$ independently and randomly. The vectors initialized in this way are for uniforming approximately on spatial, which are fixed during the quantization process and expected to compose a base set that is de-homogenized.
%%The similarity kernels $\mathbf{K}_{\mathbf{f}}$ and $\mathbf{K}_{\mathbf{r}}$ of the features $\mathbf{f}$ and noise sample set $\mathbf{r}$ are defined as presented in Definition~\ref{def:dpp}.
Then we construct the similarity kernels $\mathbf{K}_{\mathbf{f}}$ and $\mathbf{K}_{\mathbf{r}}$ of the features $\mathbf{f}$ and noise vector set $\mathbf{r}$ as the following decomposition proposed in \cite{elfeki2019gdpp}: $\mathbf{K}_{\mathbf{f}}={\phi(\mathbf{f}})^{\top} \phi(\mathbf{f})$ and $\mathbf{K}_{\mathbf{r}}={\phi(\mathbf{r})}^{\top} \phi(\mathbf{r})$, where all $\phi({f}_i)\in\phi(\mathbf{f})$ and $\phi({r}_i)\in\phi(\mathbf{r})$ are the $\ell_2$ normalized vectors that guarantees the similarity kernels $\mathbf{K}_{\mathbf{f}}$ and $\mathbf{K}_{\mathbf{r}}$ to be real positive semidefinite matrices.
The main idea of our SCI is to learn a similarity kernel $\mathbf{K}_{\mathbf{f}}$ of feature $\mathbf{f}$ and make it less than the similarity kernel $\mathbf{K}_{\mathbf{r}}$ of random data $\mathbf{r}$. 
In this way, we keep the correlation among the features of synthetic samples lower than that among random vectors, thus alleviating the homogenization of synthetic data. 
Since matching two similarity kernels directly is an unconstrained optimization problem~\cite{li2009kernel}, we construct the loss using the major characteristics: eigenvalues and eigenvectors of kernels.
Hence, our SCI loss term is defined as follows:
\begin{equation}
l_{\textrm{SCI}}=\max\left(\sum_{i}\left(\left\|\lambda_{\mathbf{f}}^{i}-\lambda_{\mathbf{r}}^{i}\right\|_{2}-\hat{\lambda}_{\mathbf{f}}^{i} \cos \left(v_{\mathbf{f}}^{i}, v_{\mathbf{r}}^{i}\right)\right), 0\right),
\end{equation}
where $\lambda_{\mathbf{f}}^{i}$ and $\lambda_{\mathbf{r}}^{i}$ are the $i^{\textrm{th}}$ eigenvalues of $\mathbf{K}_{\mathbf{f}}$ and $\mathbf{K}_{\mathbf{r}}$, respectively. $v_{\mathbf{f}}^{i}$ and $v_{\mathbf{r}}^{i}$ are the eigenvectors, and $\hat{\lambda}_{\mathbf{r}}^{i}$ is the min-max normalized version of the eigenvalues applied to alleviate the effect of noisy structures.

\subsection{Specialization for different quantization approaches}

The motivation of our DSG is to improve the performance of various generative data-free quantization methods by proposing a novel and generic data generation scheme.
Therefore, here we derive and present the most typical specializations of DSG for generative data-free PTQ and QAT approaches (Fig.~\ref{fig:specific-arch}).

%\subsubsection{DSG for Post-training Quantization}

\textbf{For the data-free PTQ approach}, DSG applies SDA to each layer for relaxed BN statistics loss terms, integrates LSE to introduce diverse degrees of attention among different samples, and uses SCI to diversify the distribution of samples in the potential input domain. 
Therefore, the overall loss $\mathcal{L}_\textrm{DSG-PTQ}$ can be defined as follows involving the above techniques:
\begin{equation}
\label{eq:LSE_expand_2}
\begin{aligned}
\mathcal{L}_\textrm{DSG-PTQ}=\left(\mathbf{X}_\textrm{LSE}\mathbf{l}_\textrm{SDA}\right)^T\mathbf{1} + l_{\textrm{SCI}}.
\end{aligned}
\end{equation}
%In addition to slack constraints for matching BN statistics, LSE provides layerwise enhancement and thus further diversifies the synthetic data.
This optimization loss $\mathcal{L}_\textrm{DSG-PTQ}$ is used to optimize synthetic data directly. After the generation process is complete, the optimized synthetic data is applied to calibrate the quantized neural network. The calibration process is usually completely separate from the generation process, and the typical calibration methods include \cite{cai2020zeroq} and \cite{guo2022squant}.
The DSG for PTQ is elaborated in Algorithm~\ref{alg:1}, and it is summarized in Fig.~\ref{fig:specific-arch} as the case (a).

%We relax constraints in matching BN statistics to bring more possibilities for the generation process. Meanwhile, we reinforce the loss of a specific layer to introduce attention for every single sample. As a result, the DSG scheme mitigates the homogenization issue in two different scales and thus generates diversified samples compared with existing methods. Finally, we use the synthetic data to calibrate the quantized network.
%Our DSG in PTQ is elaborated in Algorithm~\ref{alg:1}, and it is summarized in Fig.~\ref{fig:specific-arch} as the case (a).

\renewcommand{\algorithmicrequire}{\textbf{Input:}}
\renewcommand{\algorithmicensure}{\textbf{Output:}}
\begin{algorithm}[t]
    \small
    \caption{The process of DSG scheme in PTQ}
    \label{alg:1}
    \KwIn{pretrained model $\textrm{M}$ with $N$ BN layers, training iterations $T$.}
    \KwOut{synthetic data: $\mathbf{x}^s$, quantized network $\textrm{Q}$.}
    Initialize $\mathbf{x}^s$ from Gaussian distribution $\mathcal{N}(0,1)$\;
    Initialize $\mathbf{x}^{0}$ from Gaussian distribution $\mathcal{N}(0,1)$\;
    Get $\boldsymbol{\mu}_{i}$ and $\boldsymbol{\sigma}_{i}$ from BN layers of $\textrm{M}$, $i=1,2,\dots, N$\;
    Forward propagate $\textrm{M}(\mathbf{x}^{0})$ and gather activations\;
    Compute $\delta_i$ and $\gamma_i$ using Eq.~(\ref{eq:mn})\;
    \For{all $t=1,2,\dots, T$}
    {
    Forward propagate $\textrm{M}(\mathbf{x}^s)$ and gather activations\;
    Get $\boldsymbol{\tilde{\mu}}_{i}^s$ and $\boldsymbol{\tilde{\sigma}}_{i}^s$ from activations\;
    Compute all ${l_\textrm{SDA}}_i$ using Eq.~(\ref{eq:SDA})\;
    Descend $\mathcal{L}_\textrm{DSG-PTQ}$ using Eq.~(\ref{eq:LSE_expand_2}) and update data $\mathbf{x}^s$\;
    }
    Get synthetic data $\mathbf{x}^s$\;
    Calibrate quantized network $\textrm{Q}$ using data $\mathbf{x}^s$.
\end{algorithm}

%\subsubsection{DSG for Quantization-aware Training}

%For the data-free QAT, based on the analysis in Section~\ref{sec:homogenization} and Fig.~\ref{fig:distribution}, existing generative methods suffer homogenization at the distribution level, and also at the sample-level homogenization from both statistical and spatial perspective.
%Therefore, our DSG applies the SDA, LSE, and SCI techniques jointly in QAT to diversify the synthetic data.

%We first determine the optimization objectives for the generator $\textrm{G}$ and the quantized network $\textrm{Q}$ in the DSG scheme.
%As shown in Eq.~(\ref{eq:dfqat}), the optimization objectives in generative data-free QAT include: (1) training generator $\textrm{G}$ to synthesize better data $\mathbf{x}^s$; (2) finetuning the quantized network $\textrm{Q}$ with synthetic data $\mathbf{x}^s$ to improve accuracy.
%These two optimization goals are coupled, which leads to the high dependence of the performance of quantized network on the synthetic data.
%Therefore, for DSG in QAT, we apply an alternating optimization strategy to update the generator $\textrm{G}$ and the quantized network $\textrm{Q}$.

\textbf{For the data-free QAT approach}, DSG applies an alternating optimization strategy to update the generator $\textrm{G}$ and the quantized neural network $\textrm{Q}$.
%The generator $\textrm{G}$ is trained by exploiting the knowledge in the pre-trained full-precision network, \textit{i.e.}, supervising the generator by the information of both its BN statistics and prediction probability.
As shown in the Eq.~(\ref{eq:gen}), we initialize a set of random noise samples and specified labels, feeding it in generator $\textrm{G}$ to generate the synthetic data and then training the generator $\textrm{G}$.
The loss function $\mathcal{L}_\textrm{DSG-QAT}$ for the generator can be expressed as
\begin{equation}
\label{eq:dsg-qat-g}
%\begin{aligned}
%\mathcal{L}_\textrm{DSG-QAT}=\left(\mathbf{X}_\textrm{LSE}\mathbf{l}_\textrm{SDA}\right)^T\mathbf{1} + \mathbf{l}_{\textrm{SCI}}^T\mathbf{1}+\mathbb{E}_{{\textrm{G}({\mathbf{x}^{s*}})}, y}[\textrm{CE}(\textrm{M}(\textrm{G}({\mathbf{x}^{s*}})), y)],
\mathcal{L}_\textrm{DSG-QAT} = \left(\mathbf{X}_\textrm{LSE}\mathbf{l}_\textrm{SDA}\right)^T\mathbf{1} + \mathbf{l}_{\textrm{SCI}}^T\mathbf{1} + \mathcal{L}_\textrm{CE}^\textrm{G}(\textrm{G}),
%\end{aligned}
\end{equation}
where $\mathbf{l}_{\textrm{SCI}}=\frac{1}{B}[l_{\textrm{SCI}_1}, l_{\textrm{SCI}_2}, ..., l_{\textrm{SCI}_B}]^T$ in Eq.~(\ref{eq:dsg-qat-g}) is the SCI loss terms for intermediate features of each blocks ($B$ blocks in total) in generator $\textrm{G}$. 
$\mathcal{L}_\textrm{CE}^\textrm{G}(\textrm{G})$ is the cross-entropy loss.
%$\mathcal{L}_\textrm{CE}^\textrm{G}(\textrm{G})$ and $\beta$ (0.1 by default) are the cross-entropy loss and a trade-off parameter followed by \cite{xu2020generative}.
%where $\textrm{CE}$ is the cross-entropy loss between the label $y$ and the prediction for $\mathbf{x}$ of the full-precision network $\textrm{M}$, which constrains the synthetic data to the specified label utilizing the well-trained full-precision network. 
%This loss makes the generator $\textrm{G}$ to exploit the knowledge in the pre-trained full-precision network $\textrm{M}$, supervising the generator by the information of both its BN statistics and prediction probability.

The data synthesized by generator $\textrm{G}$ is applied to train the quantized neural network according to the usual practice~\cite{xu2020generative}, aiming to close the prediction probability distribution of the quantized neural network to that of the full-precision model.
In the quantization process, DSG train the generator $\textrm{G}$ and the quantized neural network alternately in every epoch, and the pre-trained full-precision network is fixed.
The DSG for QAT is elaborated in Algorithm~\ref{alg:2} and is further illustrated in Fig.~\ref{fig:specific-arch} as the case (b). 

\renewcommand{\algorithmicrequire}{\textbf{Input:}}
\renewcommand{\algorithmicensure}{\textbf{Output:}}
\begin{algorithm}[t]
    \small
    \caption{{The process of DSG scheme in QAT}}
    \label{alg:2}
    \KwIn{pretrained model $\textrm{M}$ with $N$ BN layers, training iterations $T$.}
    \KwOut{generator network $\textrm{G}$, quantized network $\textrm{Q}$.}
    Initialize $\mathbf{x}^{s*}$ from Gaussian distribution $\mathcal{N}(0,1)$ and specify label $y$\;
    Get $\boldsymbol{\mu}_{i}$ and $\boldsymbol{\sigma}_{i}$ from BN layers of $\textrm{M}$, $i=1,2,\dots, N$\;
    Forward propagate $\textrm{M}(\mathbf{x}^{s*})$ and gather activations\;
    Compute $\delta_i$ and $\gamma_i$ using Eq.~(\ref{eq:mn})\;
    
    \For{all $t=1,2,\dots, T$}
    {
    Generate data $(\mathbf{x}^s, y)$ and gather activations\;
    Forward propagate $\textrm{M}(\mathbf{x}^s)$ and gather activations\;
    Descend $\mathcal{L}_\textrm{DSG-QAT}$ using Eq.~(\ref{eq:dsg-qat-g}) and update $\textrm{G}$\;
    Update the quantized network $\textrm{Q}$ using data $(\mathbf{x}^s, y)$\;
    }
    Train generator network $\textrm{G}$ and generate data $(\mathbf{x}^s, y)$\;
    Train quantized network $\textrm{Q}$ using data $(\mathbf{x}^s, y)$.
\end{algorithm}

%\textbf{Optimization for the quantized network $\textrm{Q}$.} 
%In the finetuning process of the quantized network $\textrm{Q}$, we optimize the quantized network according to the usual practice~\cite{xu2020generative}.
%Specifically, the data $\mathbf{x}^s$ synthesized by generator $\textrm{G}$ is applied to optimize the quantized network $\textrm{Q}$.
%The error of the quantized network $\textrm{Q}$ is reduced to the generated data $(\mathbf{x}, y)$ and let the prediction probability distribution of $\textrm{Q}$ close to the original full-precision network $\textrm{M}$, and thereby improving the performance of the quantized network.

%In the optimization process, we train the generator $\textrm{G}$ and quantized network $\textrm{Q}$ alternately in every epoch, and the pre-trained full-precision network $\textrm{M}$ is fixed.
%In this way, we fully exploit the knowledge in the pre-trained full-precision network to improve the quality of the synthetic data and the accuracy of the quantized network. Our DSG in QAT is elaborated in Algorithm~\ref{alg:2} and is further illustrated in Fig.~\ref{fig:specific-arch} as the case (b). 

\subsection{Analysis and Discussion}

The techniques in our DSG scheme, SDA, LSE, and SCI, alleviate the homogenization of synthetic data at the distribution level and sample level. 
In Fig.~\ref{fig:distribution}, we statistically observe the behavior of the three types of data, \textit{i.e.}, data generated by our DSG scheme, data generated by other SOTA generative data-free quantization methods and real data. 
Here we analyze and discuss the effect of our DSG scheme in diversifying synthetic data in generative data-free quantization methods. 

We first analyze the synthetic data generated by our DSG scheme in PTQ.
Compared with the data generated by the existing generative data-free PTQ method (ZeroQ), the behavior of samples generated with our DSQ is more similar to that of real data statistically.  
As shown in Fig.~\ref{fig:distribution_a}, the distribution of our DSG data does not strictly fall in close vicinity to BN mean and standard deviation, instead, they are more various in value and thus more fluctuate in frequency. As introduced in Section~\ref{sec:homogenization}, the Wasserstein distance between BN statistics and our DSG synthetic data is 0.124 (compared with 0.040 of ZeroQ).
And in Fig.~\ref{fig:distribution_b}, the statistics of the data generated by ZeroQ are centralized, while that of the DSG samples are dispersed. The phenomenon means that our synthetic data significantly alleviates the sample-level homogenization of existing data-free generative PTQ methods from the statistical perspective. As mentioned in Section~\ref{sec:homogenization}, we calculate the variance of the statistics of our DSG data and compare it with ZeroQ in PTQ, which is 0.029 vs. 0.009. 
%Fig.~\ref{fig:space_e} presents the distribution of different types of synthetic data in the sample space, which shows that the behavior of our DSG samples in the sample space is close to real data and Gaussian data, that is, no additional sample-level homogenization from the spatial perspective. Considering the additional computational overhead, we combine SDA and LSE for the DSG scheme in the PTQ scenario.
Fig.~\ref{fig:space_e} presents the distribution of different types of synthetic data in the sample space.
The density index near the cluster center of the synthesized sample generated by DSG is only 1571, which is even lower than that of real data 1681 and far lower than ZeroQ 1902.

For the DSG scheme in QAT, due to the utility of SDA and LSE, the distribution of our DSG data is not constrained strictly and the synthetic samples enjoy more diverse statistics, as Fig.~\ref{fig:distribution_c} and Fig.~\ref{fig:distribution_d} show. 
The SCI further weakens the correlation among samples by constructing unsupervised loss for intermediate features in the generation process to alleviate the homogenization of the generated samples in sample space. 
%From Fig.~\ref{fig:space_f}, we observe that the spatial homogenization of the synthesized samples existing in the method (GDFQ) is alleviated significantly by our DSG scheme.
From Fig.~\ref{fig:space_f}, we observe that the density index of the synthesized sample generated by DSG is far lower than the existing data-free QAT method, which means that the sample-level homogenization is greatly alleviated from the spatial perspective.

Therefore, the data generated by our DSG could be more diverse from various perspectives comprehensively, which may have the potential to be an alternative for different quantization approaches that real data is not accessible.

\section{Experiments}

In this section, we conduct comprehensive experiments to validate the performance of our DSQ scheme on image classification tasks. 
We first conduct an ablation study to show the effects of different components, including SDA, LSE, and SCI in Section~\ref{sec:ablation}. %in the proposed method on improving the performance of quantization 
Then in Section~\ref{sec:comparison}, we compare DSQ with SOTA data-free PTQ and QAT methods respectively across various network architectures. We evaluate the DSG on CIFAR10~\cite{krizhevsky2009learning} and ImageNet (ILSVRC12)~\cite{Deng2009ImageNet} datasets in PTQ while on ImageNet dataset in QAT. 
Finally, in Section~\ref{sec:integration}, we conduct a further study on synthetic data. Specifically, we analyze and discuss the data diversity, and integrate and evaluate our synthetic data with various calibration methods and data-driven quantization methods. The results show that good diversity is an important property of high-quality synthetic data.

\textbf{DSG scheme.} 
In the data-free PTQ, the proposed DSG scheme is used for generating synthetic data while the independent calibration processes are as \cite{cai2020zeroq} and \cite{guo2022squant}, and the effectiveness is evaluated by measuring the accuracy of quantized models.
Unless otherwise specified, the calibration process for DSG in PTQ is the same as \cite{cai2020zeroq}.
For the DSG scheme in QAT, we train the generator network to synthesize the data and use it to finetune the quantized network.
The generator in the DSG scheme in QAT is constructed following ACGAN~\cite{odena2017conditional} as in \cite{xu2020generative}.

\textbf{Network architectures.}
We evaluate our DSG scheme in a wide range of network architectures with various bit-width to prove the versatility of our method. 
We employ VGG16bn~\cite{VeryDeepConvolutional}, ResNet20/18/50~\cite{he2016deep}, SqueezeNext~\cite{gholami2018squeezenext}, InceptionV3~\cite{szegedy2016rethinking}, ShuffleNet~\cite{zhang2018shufflenet}, and MobileNetV2~\cite{sandler2018mobilenetv2} with various bit-widths, including W4A4 (means 4-bit weight and 4-bit activation), W6A6, W8A8, \textit{etc}.

\textbf{Implementation details.} 
The proposed scheme is implemented by PyTorch for the sake of the powerful automatic differentiation mechanism.
We adopt Gaussian distribution as initialization for the data generation process in our DSG scheme. In our experiments, we quantize all the layers. And the activation is clipped in a layerwise manner. 
As for hyper-parameter (\textit{e.g.}, the number of iterations to provide synthetic data), we mostly follow the released official implementations or the models and settings clarified in their original paper for a fair comparison~\cite{Zhang_2021_CVPR,cai2020zeroq,xu2020generative,guo2022squant}. 
For the training and finetuning process of DSG in QAT, we use the Adam and SGD optimizer in the experiments, where momentum is 0.9 and weight decay is 1e-4.
For CIFAR10, we train quantized networks and generators for 400 epochs. The learning rates are initialized to 1e-4 and 1e-3, respectively, and both of them are decayed by 0.1 for every 100 epochs. For ImageNet, we set the initial learning rate of the quantized model as 1e-6, and other training settings are the same as those on CIFAR10.

\begin{figure}[tp]
\subfigure[PTQ]{
\includegraphics[width=0.44\linewidth]{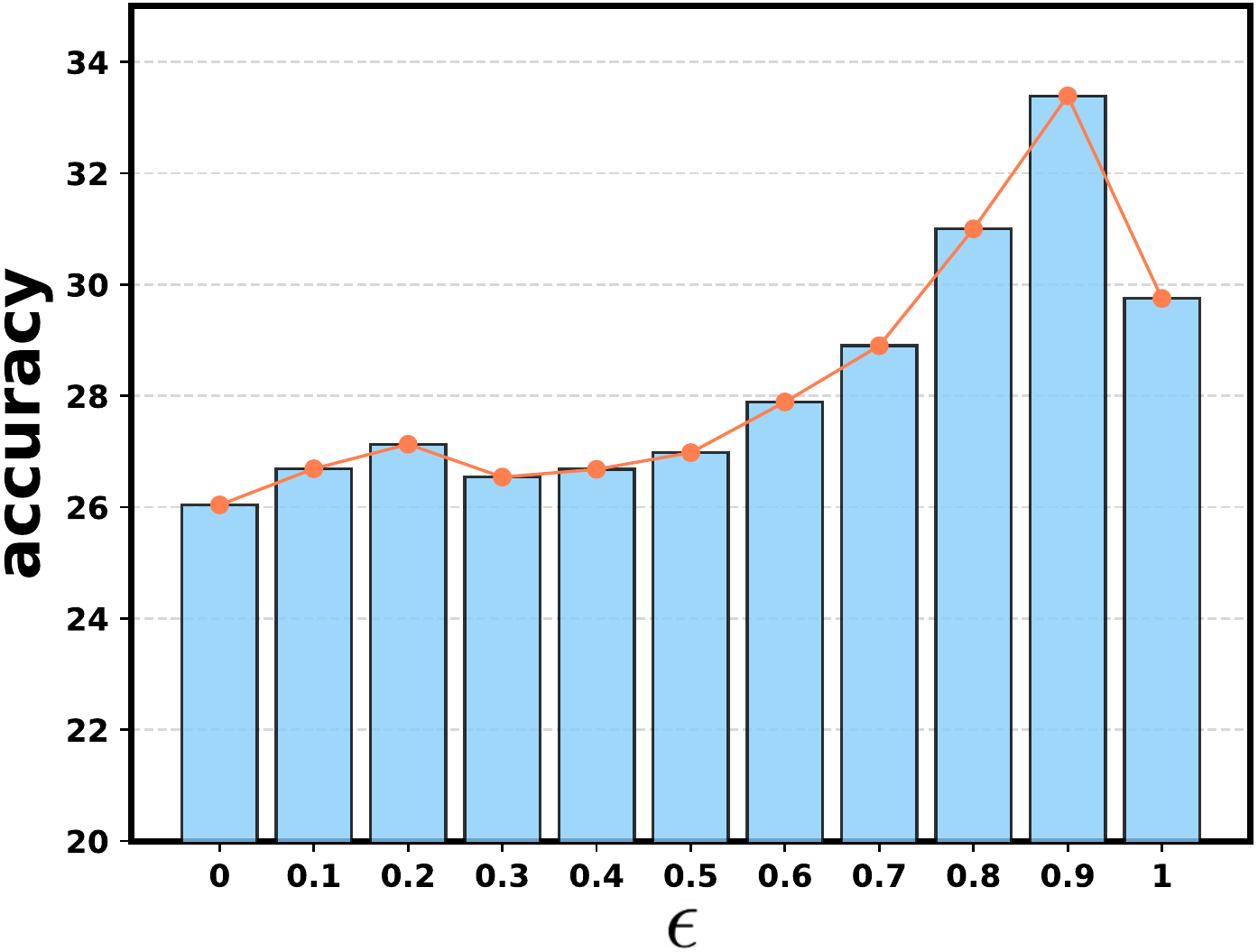}
\label{fig:Eqmn.perc_a}
}
\subfigure[QAT]{
\includegraphics[width=0.45\linewidth]{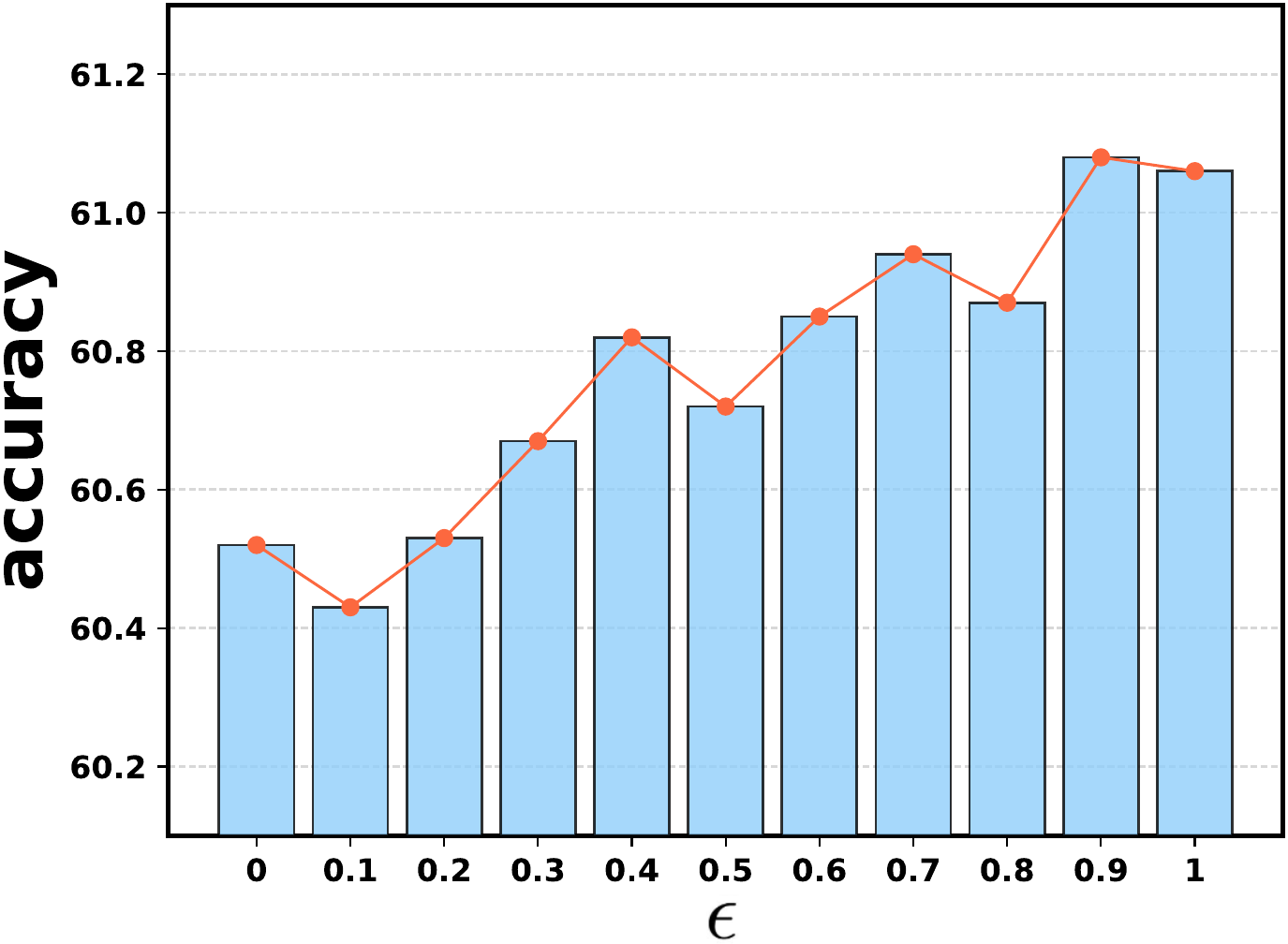}
\label{fig:Eqmn.perc_b}
}
\vspace{-0.1in}
\caption{The accuracy comparison of different $\epsilon$ values in Eq.~(\ref{eq:mn}) on ResNet18 in PTQ and QAT. 
As $\epsilon$ varies from $0$ to $0.9$, the final accuracy is mainly on the rise.
But a significant drop is caused by the outliers when $\epsilon=1$.}
\label{fig:Eqmn.perc}
\end{figure}

\begin{table}[tb]
    \caption{Ablation study for DSG scheme on ResNet18. We abbreviate quantization bits used for weights as "W-bit" (for activations as "A-bit"), top-1 test accuracy as "Top-1".}
    \label{ablation_exp}
	\centering
    \setlength{\tabcolsep}{1.mm}
    {\small
    \begin{tabular}{lcccc}
		\toprule
		{Method}  & No FT &{W-bit}  &{A-bit}  &{Top-1} \\
		\midrule
		Baseline  &-- &32 &32 &71.47  \\
		\midrule
		Vanilla (ZeroQ) &{\footnotesize\Checkmark} &4  &4  &26.04  \\
		Sample Correlation Inhibition &{\footnotesize\Checkmark} &4 &4 & 39.53 \\
		Layerwise Sample Enhancement &{\footnotesize\Checkmark} &4  &4  &27.12  \\
		Slack Distribution Alignment    &{\footnotesize\Checkmark} &4  &4  &33.39  \\
		{DSG} (Ours)    &{\footnotesize\Checkmark} &4  &4& \textbf{39.90}  \\
		%{DSG} + Sample Correlation Inhibition    &{\footnotesize\Checkmark} &4  &4& {34.51}  \\
		\midrule
		Vanilla (GDFQ) &{\footnotesize\XSolidBrush} &4  &4  &60.52  \\
		Layerwise Sample Enhancement &{\footnotesize\XSolidBrush} &4  &4  &60.71  \\
		Sample Correlation Inhibition &{\footnotesize\XSolidBrush} &4 &4 & 60.92\\
		Slack Distribution Alignment    &{\footnotesize\XSolidBrush} &4  &4  &61.08  \\
		%{DSG} (Ours)    &{\footnotesize\XSolidBrush} &4  &4& \textbf{61.58}  \\
		{DSG} (Ours)    &{\footnotesize\XSolidBrush} &4  &4& \textbf{62.18}  \\
		\bottomrule
	\end{tabular}
	}
	%\vspace{-0.2in}
\end{table}

\subsection{Ablation Study}
\label{sec:ablation}
We investigate the effect of the proposed LSE, SDA, and SCI techniques for our DSG scheme in data-free PTQ and QAT by ablation experiments. We use the ResNet18 architecture with the ImageNet dataset to evaluate our method under the W4A4 bit-width setting, which can show the effect of each part more obviously.

\subsubsection{Effect of SDA}
We first analyze the effectiveness of SDA. 
As discussed earlier, hyper-parameter $\epsilon$ determines the degree of relaxation in SDA. Therefore, in Fig.~\ref{fig:Eqmn.perc}, we verify the impact of different values of $\epsilon$ in the interval $(0,1]$ with a moderate step size 0.1, and further add $\epsilon=0$ to serve as the vanilla case that the fitting of BN statistics is constrained without relaxation.
In PTQ (Fig.~\ref{fig:Eqmn.perc_a}), when the $\epsilon$ varies from 0 to 0.9, the final performance increases gradually from 26.04\% to 33.39\%. 
In QAT (Fig.~\ref{fig:Eqmn.perc_b}), the accuracy also presents an overall upward trend (from 60.52\% to 61.08\%) as the $\epsilon$ varies from 0 to 0.9.
However, when $\epsilon=1$, the accuracy drops more or less in both scenarios, which is 3.61\% in PTQ and 0.02\% in QAT. 
These phenomena forcefully prove that by way of relaxing the constraints, the SDA method brings generated data a certain offset when fitting the BN statistic distribution, which contributes to feature diversity and consequently improves the performance. And when $\epsilon$ is set to 1, the degree of slack reaches the limit as all the outliers in $\boldsymbol{\tilde\mu}^{0}_{i}$ and $\boldsymbol{\tilde\sigma}^{0}_{i}$ are taken into account, which means the feature distribution of generated data might be out of the reasonable range. Therefore, it might result in better feature diversity in synthetic data but would harm the final accuracy of the quantized network. 
Therefore, we set the default value of $\epsilon$ as 0.9 empirically to balance divergence and the impact of outliers.

As the results shown in TABLE~\ref{ablation_exp}, the vanilla data-free PTQ method without SDA suffers a severe accuracy degradation by 7.35\%, and the SDA also provides 0.56\% accuracy promotion in data-free QAT (vanilla 60.52\% vs. SDA 61.08\%). The results reflect that SDA is essential. Compared to the other two parts of our scheme, SDA may provides a major contribution to the final performance.  

\subsubsection{Effect of LSE}

The TABLE~\ref{ablation_exp} also shows that LSE can improve the performance in both data-free QAT and PTQ.
As the ablation results show, in PTQ, LSE method gives a non-negligible increment compared with ZeroQ by 1.08\%. As for in QAT, it also helps the DSG scheme to have a slight improvement compared with the vanilla method by 0.19\%. 
%Moreover, the results for PTQ in TABLE~\ref{ablation_exp} present that the LSE and SDA are orthogonal in our DSG scheme, jointly optimizing the quantized network. In PTQ, the whole scheme including these two techniques surpasses vanilla competitors by a notable 8.49\%.  

\begin{table}[!h]
    %\vspace{0.08in}
    \caption{Results of data-free PTQ methods with ResNet20 and VGG16bn on CIFAR10.}
    \label{tb:exp_cifar}
 	\centering
 	\setlength{\tabcolsep}{1.mm}
    {\small
    \begin{tabular}{llccccc}
		\toprule
		Arch &{Method} &{No D}  &{No FT}  &{W-bit}  &{A-bit}  &\tabincell{c}{{Top-1}}\\
		\midrule
		\multirow{12}{*}{ResNet20} &Baseline   &--  &-- &32 &32 &94.08  \\
		\cmidrule{2-7}
		& Real Data  &{\footnotesize\XSolidBrush}   &{\footnotesize\Checkmark} &4  &4  &87.38  \\
		\cmidrule{2-7}
		& ZeroQ  &{\footnotesize\Checkmark}   &{\footnotesize\Checkmark} &4  &4  &85.39\\
		& {DSG} (Ours)  &{\footnotesize\Checkmark}   &{\footnotesize\Checkmark}  &4  &4 &\textbf{87.79}  \\
		\cmidrule{2-7}
		& Real Data  &{\footnotesize\XSolidBrush}   &{\footnotesize\Checkmark} &6  &6  &93.80  \\
		\cmidrule{2-7}
		& ZeroQ  &{\footnotesize\Checkmark}   &{\footnotesize\Checkmark} &6  &6  &93.33  \\
		& {DSG} (Ours)  &{\footnotesize\Checkmark}   &{\footnotesize\Checkmark}  &6  &6 &\textbf{93.55}  \\
		\cmidrule{2-7}
		& Real Data  &{\footnotesize\XSolidBrush}   &{\footnotesize\Checkmark} &8  &8  &93.95  \\%
		\cmidrule{2-7}
		& ZeroQ  &{\footnotesize\Checkmark}   &{\footnotesize\Checkmark} &8  &8  &93.94  \\
		& {DSG} (Ours)  &{\footnotesize\Checkmark}   &{\footnotesize\Checkmark}  &8  &8 &\textbf{93.97}  \\
	    \midrule
    	\multirow{12}{*}{VGG16bn} &Baseline   &--  &-- &32 &32 &93.86  \\%
    	\cmidrule{2-7}
	    &Real Data  &{\footnotesize\XSolidBrush}   &{\footnotesize\Checkmark} &4  &4  &92.50  \\
		\cmidrule{2-7}
		&ZeroQ  &{\footnotesize\Checkmark}   &{\footnotesize\Checkmark} &4  &4  &91.79  \\
	    &{DSG} (Ours)  &{\footnotesize\Checkmark}   &{\footnotesize\Checkmark}  &4  &4 &\textbf{92.89}  \\
    	\cmidrule{2-7}
    	&Real Data  &{\footnotesize\XSolidBrush}   &{\footnotesize\Checkmark} &6  &6  &93.48  \\%
    	\cmidrule{2-7}
    	&ZeroQ  &{\footnotesize\Checkmark}   &{\footnotesize\Checkmark} &6  &6  &93.45  \\%
    	&{DSG} (Ours)  &{\footnotesize\Checkmark}   &{\footnotesize\Checkmark} &6  &6  & \textbf{93.68} \\%
    	\cmidrule{2-7}
    	&Real Data  &{\footnotesize\XSolidBrush}   &{\footnotesize\Checkmark} &8  &8  &93.59  \\%
    	\cmidrule{2-7}
    	&ZeroQ  &{\footnotesize\Checkmark}   &{\footnotesize\Checkmark} &8  &8  &93.53  \\%
    	&{DSG} (Ours)  &{\footnotesize\Checkmark}   &{\footnotesize\Checkmark} &8  &8  &  \textbf{93.61}\\%
		\bottomrule
	\end{tabular}
	}
\end{table}

% table 4
\begin{table*}[!h]
    \centering
    \caption{Results of data-free PTQ methods with (a) ResNet18, ResNet50, (b) SqueezeNext, InceptionV3, and ShuffleNet on ImageNet. Here, "Arch" means the network architectures, "No D" means that none of the data is used to assist quantization, "No FT" stands for no finetuning (retraining). "Real Data" represents using real training data and quantization methods in ZeroQ (without any finetuning).}
    \subtable[ResNet18 and ResNet50]{
    \label{tb:exp_ptq_res}
    \setlength{\tabcolsep}{1.0mm}
    {\small
    \begin{tabular}{llccccc}
    \toprule
	Arch& {Method} &{No D}  &{No FT}  &{W-bit}  &{A-bit}  &\tabincell{c}{{Top-1}}\\
		\midrule
        \multirow{26}{*}{ResNet18}& Baseline   &-- &-- &32 &32 &71.47  \\
        \cmidrule{2-7}
         &Real Data &{\footnotesize\XSolidBrush}  &{\footnotesize\Checkmark} &4 &4 & 65.22 \\ % 31.86 \\
        \cmidrule{2-7}
        &DFQ &{\footnotesize\Checkmark} &{\footnotesize\Checkmark} &4 &4 &0.10 \\
        &ACIQ &{\footnotesize\Checkmark} &{\footnotesize\Checkmark} &4 &4 &7.19 \\
        &MSE &{\footnotesize\Checkmark} &{\footnotesize\Checkmark} &4 &4 &15.08 \\
        &KL &{\footnotesize\Checkmark} &{\footnotesize\Checkmark} &4 &4 &16.27 \\
        &ZeroQ &{\footnotesize\Checkmark} &{\footnotesize\Checkmark} &4 &4 &26.04 \\ 
        &SQuant &{\footnotesize\Checkmark} &{\footnotesize\Checkmark} &4 &4 & 66.14 \\
        &DSG$^1$ (Ours) &{\footnotesize\Checkmark} &{\footnotesize\Checkmark} &4 &4 &\textbf{39.90} \\
        &DSG$^2$ (Ours) & {\footnotesize\Checkmark} &{\footnotesize\Checkmark} &4 &4 &\textbf{66.67} \\
        \cmidrule{2-7}
        &Real Data &{\footnotesize\XSolidBrush}  &{\footnotesize\Checkmark} &6 &6 & 71.18 \\ %  70.62 \\
        \cmidrule{2-7}
        &ACIQ &{\footnotesize\Checkmark} &{\footnotesize\Checkmark} &6 &6 &61.15 \\
        &KL &{\footnotesize\Checkmark} &{\footnotesize\Checkmark} &6 &6 &61.34 \\
        &MSE &{\footnotesize\Checkmark} &{\footnotesize\Checkmark} &6 &6 &66.96 \\
        &DFQ &{\footnotesize\Checkmark}   &{\footnotesize\Checkmark}   &6 &6 &67.30 \\
        &ZeroQ &{\footnotesize\Checkmark} &{\footnotesize\Checkmark} &6 &6 &69.74 \\
        &DSG$^1$ (Ours) &{\footnotesize\Checkmark} &{\footnotesize\Checkmark} &6 &6 &\textbf{70.46} \\
        &DSG$^2$ (Ours) & {\footnotesize\Checkmark} &{\footnotesize\Checkmark} &6 &6 &\textbf{71.18} \\
        \cmidrule{2-7}
        &Real Data &{\footnotesize\XSolidBrush}  &{\footnotesize\Checkmark} &8 &8 & 71.48 \\ % 71.44 \\
        \cmidrule{2-7}
        &ACIQ &{\footnotesize\Checkmark} &{\footnotesize\Checkmark} &8 &8 &68.78 \\
        &DFQ &{\footnotesize\Checkmark}   &{\footnotesize\Checkmark}  &8 &8 &69.70 \\
        &KL &{\footnotesize\Checkmark} &{\footnotesize\Checkmark} &8 &8 &70.69 \\
        &MSE &{\footnotesize\Checkmark} &{\footnotesize\Checkmark} &8 &8 &71.01 \\
        &ZeroQ &{\footnotesize\Checkmark} &{\footnotesize\Checkmark}  &8 &8 &71.43 \\
        &SQuant &{\footnotesize\Checkmark} &{\footnotesize\Checkmark}  &8 &8 &71.43 \\
        &DSG$^1$ (Ours) &{\footnotesize\Checkmark} &{\footnotesize\Checkmark} &8 &8 &\textbf{71.49} \\
        &DSG$^2$ (Ours) & {\footnotesize\Checkmark} &{\footnotesize\Checkmark} &8 &8 &\textbf{71.46} \\
        \midrule
        \multirow{17}{*}{ResNet50} &Baseline   &-- &-- &32 &32 &77.72  \\ 
        \cmidrule{2-7}
        &Real Data &{\footnotesize\XSolidBrush}   &{\footnotesize\Checkmark} &4  &4  & 68.13 \\% 22.10 \\%
        \cmidrule{2-7}
        &ZeroQ &{\footnotesize\Checkmark}   &{\footnotesize\Checkmark} &4  &4  &8.20\\%
        &DFQ &{\footnotesize\Checkmark} &{\footnotesize\Checkmark} &4 &4 &10.32 \\ 
        &SQuant &{\footnotesize\Checkmark} &{\footnotesize\Checkmark} &4 &4 & 70.80\\ 
        &DSG$^1$ (Ours) &{\footnotesize\Checkmark}   &{\footnotesize\Checkmark} &4  &4  &\textbf{56.12}\\%
        &DSG$^2$ (Ours) &{\footnotesize\Checkmark}   &{\footnotesize\Checkmark} &4  &4  &\textbf{68.30}\\%
        \cmidrule{2-7}
        &Real Data &{\footnotesize\XSolidBrush}  &{\footnotesize\Checkmark} &6 &6 & 76.84 \\ % 75.52 \\
        \cmidrule{2-7}
        &OCS &{\footnotesize\XSolidBrush}  &{\footnotesize\Checkmark} &6 &6 &74.80 \\ 
        &ZeroQ &{\footnotesize\Checkmark} &{\footnotesize\Checkmark} &6 &6 &75.56 \\ 
        &SQuant &{\footnotesize\Checkmark} &{\footnotesize\Checkmark} &6 &6 & 77.05 \\ 
        &DSG$^1$ (Ours) &{\footnotesize\Checkmark} &{\footnotesize\Checkmark} &6 &6 &\textbf{76.90} \\
        &DSG$^2$ (Ours) &{\footnotesize\Checkmark} &{\footnotesize\Checkmark} &6 &6 &\textbf{77.22} \\
        \cmidrule{2-7}
        &Real Data &{\footnotesize\XSolidBrush}  &{\footnotesize\Checkmark} &8 &8 & 77.70 \\ % (squant+real data: 77.69)
        \cmidrule{2-7} 
        &ZeroQ &{\footnotesize\XSolidBrush}  &{\footnotesize\Checkmark} &8 &8 &77.67 \\ 
        %&SQuant &{\footnotesize\XSolidBrush}  &{\footnotesize\Checkmark} &8 &8 & 77.71 \\ 
        &DSG$^1$ (Ours) &{\footnotesize\Checkmark} &{\footnotesize\Checkmark} &8 &8 &\textbf{77.72} \\
        &DSG$^2$ (Ours) &{\footnotesize\Checkmark} &{\footnotesize\Checkmark} &8 &8 &\textbf{{77.83}} \\
        \bottomrule
    \end{tabular}
    }
    }
    \subtable[SqueezeNext, InceptionV3, and ShuffleNet]{
        \renewcommand\arraystretch{1.225}
        \label{tb:exp_ptq_other}
        \centering
        \setlength{\tabcolsep}{0.5mm}
        {\small
        \begin{tabular}{llccccc}
    		\toprule
    		Arch& {Method} &{No D}  &{No FT}  &{W-bit}  &{A-bit}  &\tabincell{c}{{Top-1}}\\
    		\midrule
    		\multirow{13}{*}{SqueezeNext} &Baseline   &--  &-- &32 &32 &  69.38\\%
    		\cmidrule{2-7}
    		&Real Data  &{\footnotesize\XSolidBrush}   &{\footnotesize\Checkmark} &6  &6  & 66.51 \\ % 62.88  \\%
    		\cmidrule{2-7}
    		&ZeroQ  &{\footnotesize\Checkmark}   &{\footnotesize\Checkmark} &6  &6  &39.83  \\%
    		&SQuant  &{\footnotesize\Checkmark}   &{\footnotesize\Checkmark} &6  &6  & 67.34 \\%
    		&DSG$^1$ (Ours)  &{\footnotesize\Checkmark}   &{\footnotesize\Checkmark} &6  &6  &\textbf{66.23}  \\%
    		&DSG$^2$ (Ours)  &{\footnotesize\Checkmark}   &{\footnotesize\Checkmark} &6  &6  &\textbf{68.05}  \\%
    		\cmidrule{2-7}
    		&Real Data  &{\footnotesize\XSolidBrush}   &{\footnotesize\Checkmark} &8  &8  &69.23  \\% (squant:69.20)
    		\cmidrule{2-7}
    		&ZeroQ  &{\footnotesize\Checkmark}   &{\footnotesize\Checkmark} &8  &8  &68.01  \\%
    		&SQuant  &{\footnotesize\Checkmark}   &{\footnotesize\Checkmark} &8  &8  & 69.22 \\%
    		&DSG$^1$ (Ours)  &{\footnotesize\Checkmark}   &{\footnotesize\Checkmark} &8  &8  &\textbf{69.27}  \\%
    		&DSG$^2$ (Ours)  &{\footnotesize\Checkmark}   &{\footnotesize\Checkmark} &8  &8  &\textbf{{69.37}}  \\%
    		\midrule
    		\multirow{18}{*}{InceptionV3} &Baseline   &--  &-- &32 &32 &  78.80\\%
    		\cmidrule{2-7}
    		&Real Data  &{\footnotesize\XSolidBrush}   &{\footnotesize\Checkmark} &4  &4  & 73.50 \\ % 23.23 \\%
    		\cmidrule{2-7}
    		&ZeroQ  &{\footnotesize\Checkmark}   &{\footnotesize\Checkmark} &4  &4  & 12.00 \\%
    		&SQuant  &{\footnotesize\Checkmark}   &{\footnotesize\Checkmark} &4  &4  & 73.26 \\%
    		&DSG$^1$ (Ours)  &{\footnotesize\Checkmark}   &{\footnotesize\Checkmark} &4  &4  & \textbf{57.17} \\%
    		&DSG$^2$ (Ours)  &{\footnotesize\Checkmark}   &{\footnotesize\Checkmark} &4  &4  &\textbf{74.02}  \\%
    		\cmidrule{2-7}
    		&Real Data  &{\footnotesize\XSolidBrush}   &{\footnotesize\Checkmark} &6  &6  & 78.59 \\ %  77.96  \\%
    		\cmidrule{2-7}
    		&ZeroQ  &{\footnotesize\Checkmark}   &{\footnotesize\Checkmark} &6  &6  & 75.14 \\%
    		&SQuant  &{\footnotesize\Checkmark}   &{\footnotesize\Checkmark} &6  &6  & 78.30 \\%
    		&DSG$^1$ (Ours)  &{\footnotesize\Checkmark}   &{\footnotesize\Checkmark} &6  &6  & \textbf{78.12} \\%
    		&DSG$^2$ (Ours)  &{\footnotesize\Checkmark}   &{\footnotesize\Checkmark} &6  &6  &\textbf{78.59}  \\%
    		\cmidrule{2-7}
    		&Real Data  &{\footnotesize\XSolidBrush}   &{\footnotesize\Checkmark} &8  &8  & 78.79 \\ % 78.78  \\%
    		\cmidrule{2-7}
    		&ZeroQ  &{\footnotesize\Checkmark}   &{\footnotesize\Checkmark} &8  &8  & 78.70 \\%
    		&SQuant  &{\footnotesize\Checkmark}   &{\footnotesize\Checkmark} &8  &8  & 78.79 \\%
    		&DSG$^1$ (Ours)  &{\footnotesize\Checkmark}   &{\footnotesize\Checkmark} &8  &8  & \textbf{78.81}  \\%
    		&DSG$^2$ (Ours)  &{\footnotesize\Checkmark}   &{\footnotesize\Checkmark} &8  &8  &\textbf{78.85}  \\%
    		\midrule
    		\multirow{13}{*}{ShuffleNet} &Baseline   &--  &-- &32 &32 & 65.07 \\%
    		\cmidrule{2-7}
    		&Real Data  &{\footnotesize\XSolidBrush}   &{\footnotesize\Checkmark} &6  &6  & 56.25 \\ % 44.75 \\%
    		\cmidrule{2-7}
    		&ZeroQ  &{\footnotesize\Checkmark}   &{\footnotesize\Checkmark} &6  &6  & 39.92\\%
    		&SQuant  &{\footnotesize\Checkmark}   &{\footnotesize\Checkmark} &6  &6  & 60.25 \\%
    		&DSG$^1$ (Ours)  &{\footnotesize\Checkmark}   &{\footnotesize\Checkmark} &6  &6  &  \textbf{60.71}\\%
    		&DSG$^2$ (Ours)  &{\footnotesize\Checkmark}   &{\footnotesize\Checkmark} &6  &6  &\textbf{61.94}  \\%
    		\cmidrule{2-7}
    		&Real Data  &{\footnotesize\XSolidBrush}   &{\footnotesize\Checkmark} &8  &8  & 64.52 \\% (squant+real data: 64.26)
    		\cmidrule{2-7}
    		&ZeroQ  &{\footnotesize\Checkmark}   &{\footnotesize\Checkmark} &8  &8  & 64.46 \\%
    		&SQuant  &{\footnotesize\Checkmark}   &{\footnotesize\Checkmark} &8  &8  & 64.68 \\%
    		&DSG$^1$ (Ours) &{\footnotesize\Checkmark}   &{\footnotesize\Checkmark} &8  &8  & \textbf{64.87} \\%
    		&DSG$^2$ (Ours)  &{\footnotesize\Checkmark}   &{\footnotesize\Checkmark} &8  &8  &\textbf{64.97}  \\%
    		\bottomrule
    	\end{tabular}
    	}
    }
    \label{tb:exp_ptq_imagenet}
\end{table*}

\subsubsection{Effect of SCI}
Then we evaluate SCI in data-free quantization. 
The motivation of our SCI is to alleviate the sample-level homogenization from the spatial perspective.
TABLE~\ref{ablation_exp} shows the effects of using SCI. In QAT, compared with the vanilla method, integrating with SCI helps to acquire better accuracy, which is 0.40\% higher. Also, in PTQ, the SCI method gains 13.49\% accuracy compared with using pure random inputs. 
% In the previous analysis of Fig.~\ref{fig:distribution}, we present that spatial homogenization exists in the synthetic data of the generative data-free QAT methods, so the motivation of our SCI is to alleviate the sample-level homogenization from the spatial perspective.
% And the application of SCI in PTQ achieves few performance gains (only 0.05\%), which also verifies our previous observation that the vanilla data-free PTQ method (ZeroQ) does not seem to cause obvious spatial homogenization, so optimizing the synthetic samples from the spatial perspective has just little effect.
% The results show that the proposed SCI alleviates the sample-level homogenization from the spatial perspective and improves performance significantly in generative data-free QAT.

Intuitively, these three techniques are motivated by different observations, and also the processes are carefully engineered that they hardly interfere with each other. 
In short, SDA and LSE improve the loss related to BN statistics, aiming to prevent the generated samples from overfitting to BN statistics and makes the samples focus on the statistics of different layers. While the SCI for data-free quantization constructs an unsupervised loss to constrain the features in the generator to holding distances among them, so that mitigate the sample homogenization. 
The results show that, in data-free QAT, DSG scheme equipped with these techniques obtains 1.66\% improvement in total with ResNet18 under W4A4 setting, which is up to 62.18\%. As for data-free PTQ, DSG can also boost the performance to 13.86\% with all the techniques. 

\subsection{Comparison with SOTA Methods}
\label{sec:comparison}
We extensively evaluate our DSG scheme on a wide range of architectures on CIFAR10 and ImageNet datasets for the image classification tasks. 
We denote the bit-width setting of the quantized network as W$w$A$a$ where $w$ is the bit-width for weight and $a$ is that for activation, like W8A8, W6A6, W4A4, \textit{etc}.

\subsubsection{Comparison with Data-free PTQ Methods}

% table 5
\begin{table*}[tb]
    \centering
    \caption{Results of data-free QAT methods with ResNet18, ResNet50, InceptionV3, ShuffleNet, and MobileNetV2 on ImageNet.}
    \subtable[ResNet18 and ResNet50]{
    \renewcommand\arraystretch{1.498}
    \label{tb:exp_qat_res}
    \setlength{\tabcolsep}{0.5mm}
    {\small
    \begin{tabular}{llccccc}
    \toprule
		Arch &{Method} &{No D}  &{No FT}  &{W-bit}  &{A-bit}  &Top-1\\
		\midrule
        \multirow{15}{*}{ResNet18} &Baseline   &-- &-- &32 &32 &71.74  \\ 
        \cmidrule{2-7}
    	&Real Data  &{\footnotesize\XSolidBrush}   &{\footnotesize\XSolidBrush} &4  &4  & 63.87 \\%
        \cmidrule{2-7}
        &GDFQ &{\footnotesize\Checkmark} &{\footnotesize\XSolidBrush} &4 &4 &60.52 \\
        %&DSG (Ours) &{\footnotesize\Checkmark} &{\footnotesize\XSolidBrush} &4 &4 & \textbf{61.58}\\
        &DSG (Ours) &{\footnotesize\Checkmark} &{\footnotesize\XSolidBrush} &4 &4 & \textbf{62.18}\\
        \cmidrule{2-7}
    	&Real Data  &{\footnotesize\XSolidBrush}   &{\footnotesize\XSolidBrush} &6  &6  & 71.42 \\%
        \cmidrule{2-7} 
        &Integer-Only &{\footnotesize\XSolidBrush}  &{\footnotesize\XSolidBrush} &6 &6 &67.30 \\ 
        &GDFQ &{\footnotesize\Checkmark} &{\footnotesize\XSolidBrush} &6 &6 &70.43 \\
        %&DSG (Ours) &{\footnotesize\Checkmark} &{\footnotesize\XSolidBrush} &6 &6 &\textbf{70.83} \\ 
        &DSG (Ours) &{\footnotesize\Checkmark} &{\footnotesize\XSolidBrush} &6 &6 &\textbf{71.12} \\
        \cmidrule{2-7}
    	&Real Data  &{\footnotesize\XSolidBrush}   &{\footnotesize\XSolidBrush} &8  &8  & 71.44 \\%
        \cmidrule{2-7}
        &DFC &{\footnotesize\Checkmark}   &{\footnotesize\XSolidBrush} &8 &8 &69.57 \\ 
        &RVQuant &{\footnotesize\XSolidBrush} &{\footnotesize\XSolidBrush} &8 &8 &70.01 \\
        &GDFQ &{\footnotesize\Checkmark} &{\footnotesize\XSolidBrush} &8 &8 &71.43 \\
        %&DSG (Ours) &{\footnotesize\Checkmark} &{\footnotesize\XSolidBrush} &8 &8 &\textbf{71.43} \\
        &DSG (Ours) &{\footnotesize\Checkmark} &{\footnotesize\XSolidBrush} &8 &8 &\textbf{71.54} \\
        \midrule
        \multirow{15}{*}{ResNet50} &Baseline   &-- &-- &32 &32 &77.72  \\ 
        \cmidrule{2-7}
    	&Real Data  &{\footnotesize\XSolidBrush}   &{\footnotesize\XSolidBrush} &4  &4  & 70.27 \\%
        \cmidrule{2-7}
        &GDFQ &{\footnotesize\Checkmark} &{\footnotesize\XSolidBrush} &4 &4 &55.65 \\
        &RVQuant &{\footnotesize\XSolidBrush} &{\footnotesize\XSolidBrush} &4 &4 &64.90 \\
        &ZAQ &{\footnotesize\Checkmark} &{\footnotesize\XSolidBrush} &4 &4 &70.06 \\
        %&DSG (Ours) &{\footnotesize\Checkmark} &{\footnotesize\XSolidBrush} &4 &4 &\textbf{70.12} \\ 
        &DSG (Ours) &{\footnotesize\Checkmark} &{\footnotesize\XSolidBrush} &4 &4 &\textbf{71.96} \\ 
        \cmidrule{2-7}
    	&Real Data  &{\footnotesize\XSolidBrush}   &{\footnotesize\XSolidBrush} &6  &6  & 77.56 \\%
        \cmidrule{2-7}
        &GDFQ &{\footnotesize\Checkmark} &{\footnotesize\XSolidBrush} &6 &6 &76.59 \\
        &ZS-CGAN &{\footnotesize\Checkmark} &{\footnotesize\XSolidBrush} &6 &6 &76.82 \\
        %&DSG (Ours) &{\footnotesize\Checkmark} &{\footnotesize\XSolidBrush} &6 &6 &\textbf{77.10} \\ 
        &DSG (Ours) &{\footnotesize\Checkmark} &{\footnotesize\XSolidBrush} &6 &6 &\textbf{77.25} \\ 
        \cmidrule{2-7}
    	&Real Data  &{\footnotesize\XSolidBrush}   &{\footnotesize\XSolidBrush} &8  &8  & 77.66 \\%
        \cmidrule{2-7}
        &GDFQ &{\footnotesize\Checkmark} &{\footnotesize\XSolidBrush} &8 &8 &77.51 \\
        %&DSG (Ours) &{\footnotesize\Checkmark} &{\footnotesize\XSolidBrush} &8 &8 &\textbf{77.62} \\ 
        &DSG (Ours) &{\footnotesize\Checkmark} &{\footnotesize\XSolidBrush} &8 &8 &\textbf{77.64} \\ 
         \bottomrule
    \end{tabular}
    }
    }
    \subtable[ShuffleNet, MobileNetV2, and InceptionV3]{
        \renewcommand\arraystretch{1}
        \label{tb:exp_qat_other}
        \centering
        \setlength{\tabcolsep}{0.5mm}
        {\small
        \begin{tabular}{llccccc}
    		\toprule
    		Arch &{Method} &{No D}  &{No FT}  &{W-bit}  &{A-bit}  &\tabincell{c}{{Top-1}}\\
    		\midrule
    		\multirow{12}{*}{ShuffleNet}&Baseline   &--  &-- &32 &32 &65.07  \\%
    		\cmidrule{2-7}
    		&Real Data  &{\footnotesize\XSolidBrush}   &{\footnotesize\XSolidBrush} &4  &4  & 29.18 \\%
            \cmidrule{2-7}
    		&GDFQ  &{\footnotesize\Checkmark}   &{\footnotesize\XSolidBrush} &4  &4  &21.78 \\%
    		%&{DSG} (Ours)  &{\footnotesize\Checkmark}   &{\footnotesize\XSolidBrush} &4  &4  &\textbf{28.11} \\%
    		&{DSG} (Ours)  &{\footnotesize\Checkmark}   &{\footnotesize\XSolidBrush} &4  &4  &\textbf{29.71} \\%
            \cmidrule{2-7}
    		&Real Data  &{\footnotesize\XSolidBrush}   &{\footnotesize\XSolidBrush} &6  &6  & 62.89 \\%
    		\cmidrule{2-7}
    		&GDFQ  &{\footnotesize\Checkmark}   &{\footnotesize\XSolidBrush} &6  &6  &60.12 \\%
    		%&{DSG} (Ours)  &{\footnotesize\Checkmark}   &{\footnotesize\XSolidBrush} &6  &6  &\textbf{61.32}  \\%
    		&{DSG} (Ours)  &{\footnotesize\Checkmark}   &{\footnotesize\XSolidBrush} &6  &6  &\textbf{61.37}  \\%
            \cmidrule{2-7}
    		&Real Data  &{\footnotesize\XSolidBrush}   &{\footnotesize\XSolidBrush} &8  &8  & 62.95 \\%
    		\cmidrule{2-7}
    		&GDFQ  &{\footnotesize\Checkmark}   &{\footnotesize\XSolidBrush} &8  &8  &64.03  \\%
    		%&{DSG} (Ours)  &{\footnotesize\Checkmark}   &{\footnotesize\XSolidBrush} &8  &8  &\textbf{64.51}  \\%
    		&{DSG} (Ours)  &{\footnotesize\Checkmark}   &{\footnotesize\XSolidBrush} &8  &8  &\textbf{64.76}  \\%
    		\midrule
    		\multirow{15}{*}{MobileNetV2}&Baseline   &--  &-- &32 &32 &71.88  \\%
            \cmidrule{2-7}
    		&Real Data  &{\footnotesize\XSolidBrush}   &{\footnotesize\XSolidBrush} &4  &4  & 66.39 \\%
    		\cmidrule{2-7}
    		&GDFQ  &{\footnotesize\Checkmark}   &{\footnotesize\XSolidBrush} &4  &4  &51.30 \\%
    		%&{DSG} (Ours)  &{\footnotesize\Checkmark}   &{\footnotesize\XSolidBrush} &4  &4  &\textbf{54.66}  \\%
    		&{DSG} (Ours)  &{\footnotesize\Checkmark}   &{\footnotesize\XSolidBrush} &4  &4  &\textbf{60.46}  \\%
            \cmidrule{2-7}
    		&Real Data  &{\footnotesize\XSolidBrush}   &{\footnotesize\XSolidBrush} &6  &6 & 72.11 \\%
    		\cmidrule{2-7}
    		&Integer-Only  &{\footnotesize\XSolidBrush}   &{\footnotesize\XSolidBrush} &6  &6  &70.90 \\%
    		&GDFQ  &{\footnotesize\Checkmark}   &{\footnotesize\XSolidBrush} &6  &6  &70.98 \\%
    		&GZNQ  &{\footnotesize\Checkmark}   &{\footnotesize\XSolidBrush} &6  &6  &71.12 \\%
    		%&{DSG} (Ours)  &{\footnotesize\Checkmark}   &{\footnotesize\XSolidBrush} &6  &6  &\textbf{71.22}  \\%
    		&{DSG} (Ours)  &{\footnotesize\Checkmark}   &{\footnotesize\XSolidBrush} &6  &6  &\textbf{71.48}  \\%
            \cmidrule{2-7}
    		&Real Data  &{\footnotesize\XSolidBrush}   &{\footnotesize\XSolidBrush} &8  &8  & 72.92\\%
    		\cmidrule{2-7}
    		&GDFQ  &{\footnotesize\Checkmark}   &{\footnotesize\XSolidBrush} &8  &8  &  70.17\\%
    		&ZAQ  &{\footnotesize\Checkmark}   &{\footnotesize\XSolidBrush} &8  &8  &  71.43\\%
    		%&{DSG} (Ours)  &{\footnotesize\Checkmark}   &{\footnotesize\XSolidBrush} &8  &8  &\textbf{72.89}  \\%
    		&{DSG} (Ours)  &{\footnotesize\Checkmark}   &{\footnotesize\XSolidBrush} &8  &8  &\textbf{72.90}  \\%
    		\midrule
    		\multirow{12}{*}{InceptionV3} &Baseline   &--  &-- &32 &32 &78.80  \\%
            \cmidrule{2-7}
    		&Real Data  &{\footnotesize\XSolidBrush}   &{\footnotesize\XSolidBrush} &4  &4  & 73.51 \\%
    		\cmidrule{2-7}
    		&GDFQ  &{\footnotesize\Checkmark}   &{\footnotesize\XSolidBrush} &4  &4  &70.39 \\%
    		%&{DSG} (Ours)  &{\footnotesize\Checkmark}   &{\footnotesize\XSolidBrush} &4  &4  &\textbf{71.72}  \\%
    		&{DSG} (Ours)  &{\footnotesize\Checkmark}   &{\footnotesize\XSolidBrush} &4  &4  &\textbf{72.01}  \\%
            \cmidrule{2-7}
    		&Real Data  &{\footnotesize\XSolidBrush}   &{\footnotesize\XSolidBrush} &6  &6  & 78.81 \\%
    		\cmidrule{2-7}
    		&GDFQ  &{\footnotesize\Checkmark}   &{\footnotesize\XSolidBrush} &6  &6  &77.20 \\%4
    		%&{DSG} (Ours)  &{\footnotesize\Checkmark}   &{\footnotesize\XSolidBrush} &6  &6  &\textbf{78.51}  \\%
    		&{DSG} (Ours)  &{\footnotesize\Checkmark}   &{\footnotesize\XSolidBrush} &6  &6  &\textbf{78.60}  \\%
            \cmidrule{2-7}
    		&Real Data  &{\footnotesize\XSolidBrush}   &{\footnotesize\XSolidBrush} &8  &8  & 79.00  \\%
    		\cmidrule{2-7}
    		&GDFQ  &{\footnotesize\Checkmark}   &{\footnotesize\XSolidBrush} &8  &8  &78.62  \\%
    		%&{DSG} (Ours)  &{\footnotesize\Checkmark}   &{\footnotesize\XSolidBrush} &8  &8  &\textbf{78.94}  \\%
    		&{DSG} (Ours)  &{\footnotesize\Checkmark}   &{\footnotesize\XSolidBrush} &8  &8  &\textbf{78.94}  \\%
    		\bottomrule
    	\end{tabular}
    }
    }
    \label{tb:exp_qat_imagenet}
\end{table*}

To evaluate the advantage of the proposed scheme in PTQ, we first compare our DSG against other data-free PTQ methods (ZeroQ~\cite{cai2020zeroq}, DFQ~\cite{Nagel_2019_ICCV}, ACIQ~\cite{ACIQ}, MSE~\cite{chen2015mxnet}, KL~\cite{sung2015resiliency}, SQuant~\cite{guo2022squant}, and OCS~\cite{zhao2019improving}) on CIFAR10 and ImageNet datasets. Among these methods, ZeroQ and SQuant are typical generative data-free PTQ methods, which reconstruct synthetic data and calibrate the quantized network in different ways. Thus, we evaluate the generation performance of our DSG in PTQ with calibration methods from ZeroQ and SQuant, denoted as DSG$^1$ and DSG$^2$.
Other quantization methods use weight equalization or analytical clip range to improve the network performance. Besides, we additionally compare our method with OCS, which is also a PTQ method but requires real data for calibration. We evaluate these methods on various bit-width settings, and the results on CIFAR10 and ImageNet dataset are shown as TABLE~\ref{tb:exp_cifar} and TABLE~\ref{tb:exp_ptq_imagenet}, respectively.

Specifically, on CIFAR10~\cite{krizhevsky2009learning} dataset, we evaluate our DSG with ResNet20~\cite{7780459} and VGG16bn~\cite{VeryDeepConvolutional}. 
The results are shown in TABLE~\ref{tb:exp_cifar}. 
Under all settings on the CIFAR10 dataset, our DSG far exceeds existing SOTA methods. And we highlight that the quantized network calibrated with DSG data completely surpasses the network calibrated with real data on various settings. For example, under the W4A4 settings of ResNet20 and VGG6bn, our method exceeds the calibration with real data by 0.41\% and 0.39\%, respectively.

As listed in TABLE~\ref{tb:exp_ptq_imagenet}, the results over various network architectures, including ResNet18/50~\cite{he2016deep}, SqueezeNext~\cite{gholami2018squeezenext}, InceptionV3~\cite{szegedy2016rethinking}, and ShuffleNet~\cite{zhang2018shufflenet}, demonstrate that our proposed DSG significantly outperforms previous sample generation methods.
As the results shown, the accuracy of quantized networks calibrated with our DSG data under the W8A8 setting is almost not decreased and even surpasses the full-precision (W32A32) baseline networks on the ResNet18 and InceptionV3 architectures. 
The higher accuracy might be attributed to making full use of the potential of the quantized network. Quantizing the network to W8A8 maintains greater representation ability of the network compared with the lower bit-width quantization (such as W4A4) and brings less quantization error. Thus, under this setting, our diversified data can alleviate the performance degradation of the network quantization while even leading to a better solution compared with the full-precision network.
And the advantage of our DSG gets more evident when the bit-width becomes lower. For example, our DSG outperforms ZeroQ on SqueezeNext on the W6A6 setting by more than 26\%, and even outperforms real data by 1.54\%. 
And it is noteworthy that in W4A4 cases with ResNet18 architecture, our DSG surpasses ZeroQ by 13.86\%, and even surpasses real data by a notable 1.45\% with the compared SQuant calibration, which is up to 66.67\%. Moreover, with InceptionV3 under W4A4 settings, our DSG helps to acquire 74.02\% accuracy, which is higher than SQuant and real data by eminent 0.76\% and 0.52\%, respectively.

\subsubsection{Comparison with Data-free QAT Methods}

Furthermore, to demonstrate the applicability of our DSG scheme in data-free QAT, we compare it with existing QAT methods, such as DFC~\cite{haroush2020knowledge}, RVQuant~\cite{park2018valueaware}, Integer-Only~\cite{jacob2018quantization}, ZAQ~\cite{yuang2021zaq}, GZNQ~\cite{he2021generative}, ZS-CGAN~\cite{choi2021zero}, and GDFQ~\cite{xu2020generative}. These methods are engineered to finetune and update the network parameters. DFC is a finetuning method to recover accuracy for ultra-low bit-width cases, which uses Inceptionism~\cite{Mordvintsev2015InceptionismGD} to facilitate the generation of data with random labels. The other mentioned methods finetune the quantized network to improve the accuracy of the network. Among them, ZAQ and GDFQ introduce a generator to synthesize data and use the generated data to finetune the network. 
We evaluate these methods on various bit-width settings in the more challenging large-scale image classification task. The results on ImageNet are shown as TABLE~\ref{tb:exp_qat_imagenet}.

We test these methods on ResNet18/50~\cite{he2016deep}, InceptionV3~\cite{szegedy2016rethinking}, ShuffleNet~\cite{zhang2018shufflenet}, and MobileNetV2~\cite{sandler2018mobilenetv2}. The results in TABLE~\ref{tb:exp_qat_imagenet} show that our DSG enjoys the best performance. Concretely, as can be seen, regardless of network architecture, our DSG method surpasses many previous methods, including DFC, RVQuant, Integer-Only, and GDFQ in various bit-width settings. 
And it is noteworthy that, after finetuning with our DSG data, quantized ResNet18 and InceptionV3 even surpass their full-precision counterparts under the W8A8 setting. 
%But the improvement is more obvious in lower bit-width settings, such as 4-bit, our DSG scheme outperforms real data by up to 11.66\% with InceptionV3, and that by 2.67\% with ResNet18. 
Comparing with other data-free QAT training methods under W4A4 setting, our DSG outperforms GDFQ by 1.66\% and 16.31\% with ResNet18 and ResNet50 respectively, also higher than RVQuant with real data by 7.06\% with ResNet50. 
Besides the mainstream neural architectures, we also evaluate our DSG over existing well-designed lightweight architectures.
The proposed DSG shows an overwhelming advantage over GDFQ with a similar training pipeline, 9.16\% higher with MobileNetV2 under W4A4, 7.93\% higher with ShuffleNet under W4a4, and 1.62\% higher with InceptionV3 under W4A4. 

In short, our DSG scheme outperforms other competitors over a wide range of experiments in data-free PTQ and QAT, including various bit-width, different network architectures, and two datasets. All the results forcefully demonstrate that the diversity of generated data is significant to calibrate the model for higher accuracy, especially in ultra-low bit-width conditions. If the synthetic data are homogeneous or even identical at the distribution and sample level, it would be ineffectual when used to quantize the network. Our DSG effectively diversifies the synthetic data and thus improves the performance of the quantized neural network.

\subsection{Further Study on Synthetic Data}
\label{sec:integration}
To further demonstrate the data diversifying caused by our DSG scheme has a general gain to the quantized network, we provide more experiments as corroborations to support our viewpoint.
In our paper, we have conducted a bunch of experiments on our synthetic data to evaluate its effectiveness in improving the network performance, including analyzing the data diversity and integrating the synthetic data generated by our method to different calibration methods and data-driven quantization methods.
The experiments show that benefits by the increase in the diversity of synthetic data, the network performance in various methods is significantly improved. It shows that good diversity is an important property of high-quality synthetic data.

\subsubsection{Analysis and Discussion for Data Diversity}
In this section, we further discuss the visualization results of our DSG scheme.  First, for data-free PTQ, we show the distribution of statistics of the real data, vanilla data (ZeroQ), and DSG data (ours) in Fig.~\ref{fig.homogenization}, which explains that SDA and LSE alleviate the homogenization from the aforementioned two levels.
Then, for data-free QAT, we visualized samples of vanilla data (GDFQ), DSG data (ours), real data, and Gaussian data in Fig.~\ref{fig:dsg-qat} to show the additional effect of SCI that diversifies samples by inhibiting feature correlation.
We also visualize the synthetic samples of our method and other generative data-free quantization methods (ZeroQ, GDFQ) in Fig.~\ref{fig:visualization} to visually show the effect of data diversifying.

\begin{figure}[t]
\centering
%\vspace{-0.2in}
\includegraphics[width=0.99\linewidth]{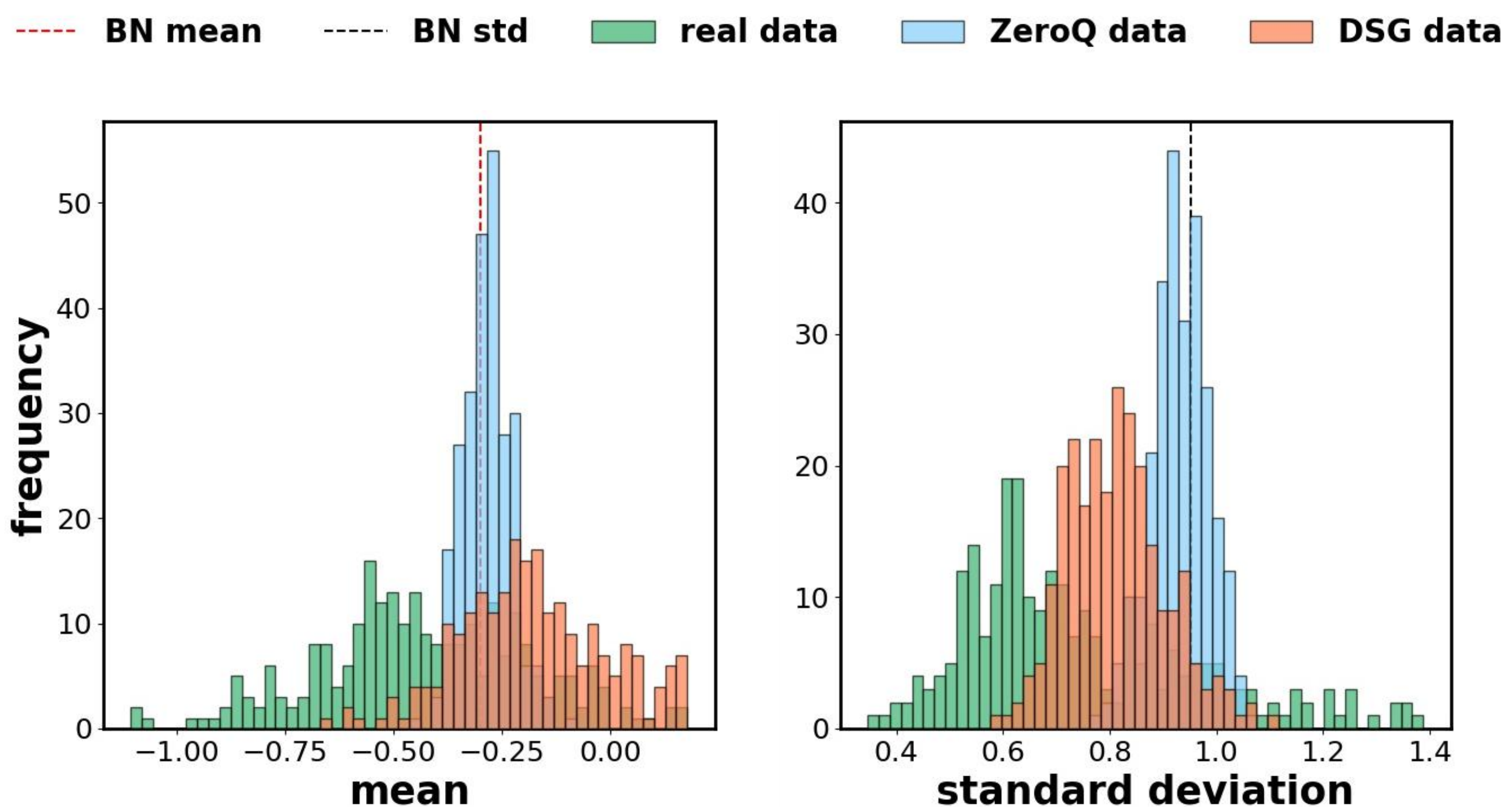}
\caption{Mean and standard deviation of the activations in one channel of ResNet18 when feeding different types of data (with 256 samples), including the synthetic data generated by ZeroQ and DSG and the real data. Each sample generated by ZeroQ behaves similarly overfitting BN statistics compared with real data, which shows the homogenization at both distribution and sample level. Our DSG data enjoys the diversity close to real data to obtain the accurate quantized network.}
\label{fig.homogenization}
\end{figure}

In Fig.~\ref{fig.homogenization}, it can be obviously investigated that DSQ samples behave more like real data than vanilla data on the offset of mean and standard deviation statistics, which corresponds to the diversity at the distribution level of our generated data scheme. Especially, the SDA plays an important role to make the distribution diverse by slacking the constraint of statistics during the generation process.
This phenomenon proves that our DSG scheme diversifies the synthetic data at the distribution level. 
Moreover, the DSG scheme also generates data with a larger variance compared to the vanilla scheme, which implies that our data samples are widely dispersed and more in line with the real situation. Especially, both SDA and LSE jointly promote diversity at the sample level, which might be useful in providing more content information.

\begin{figure}[t]
\centering
%\vspace{-0.1in}
\includegraphics[width=0.99\linewidth]{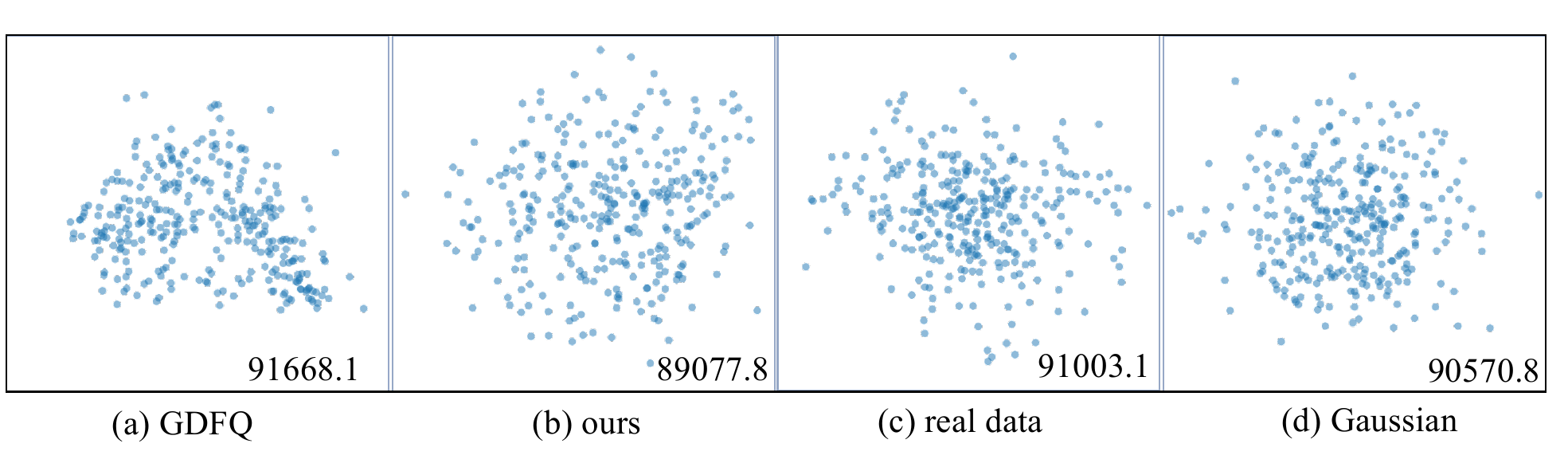}
\vspace{-0.1in}
\caption{Visualization of samples generated by (a) GDFQ method, (b) our SCI method, and (c) real data, (d) Gaussian random data. We collect 512 samples for each type of data randomly and reduce dimensions with PCA. Samples generated by our SCI method are more scattered than GDFQ samples and close to that of real data.}
%\vspace{-0.1in}
\label{fig:dsg-qat}
\end{figure}

The motivation of the SCI method is to inhibit the correlation among features in the generator, and thus the synthetic samples are separated, as shown in Fig.~\ref{fig:dsg-qat}. 
We visualize the data generated by the vanilla method (GDFQ) and our SCI method, as well as the real data and Gaussian random data. At the right-bottom corner of each subplot, 
we measure the diversity of each set of features by summarizing the elements in the similarity kernel $\mathbf{K_f}$ as an index $s=\sum_{i,j\in \mathbf{f}_k} [\mathbf{K}_{i,j}]$, which indicates the feature distance between samples in spatial perspective. The larger the $s$, the more similar the samples.
As can be seen from Fig.~\ref{fig:dsg-qat}, the Gaussian random data is dispersive and stochastic. And it is also quite obvious that the real data is naturally distinct. However, samples generated by the vanilla GDFQ method, as shown in the subplot (a), are generally more concentrated in comparison, some of which are even overlapped. So it has the biggest $s$. 
Instead, our SCI utilizes the random Gaussian data as initialization and inhibits the correlation between samples to make the data more dispersed than Gaussian sampling. 
The samples generated by our method, which has the smallest $s$, are more dispersed than those generated by the vanilla method and seem close to real data. It demonstrates that the SCI method helps to avoid homogenization from the spatial perspective, and thus diversifies the synthetic samples from each other.

\begin{figure*}[t]
\centering
\includegraphics[width=0.8\linewidth]{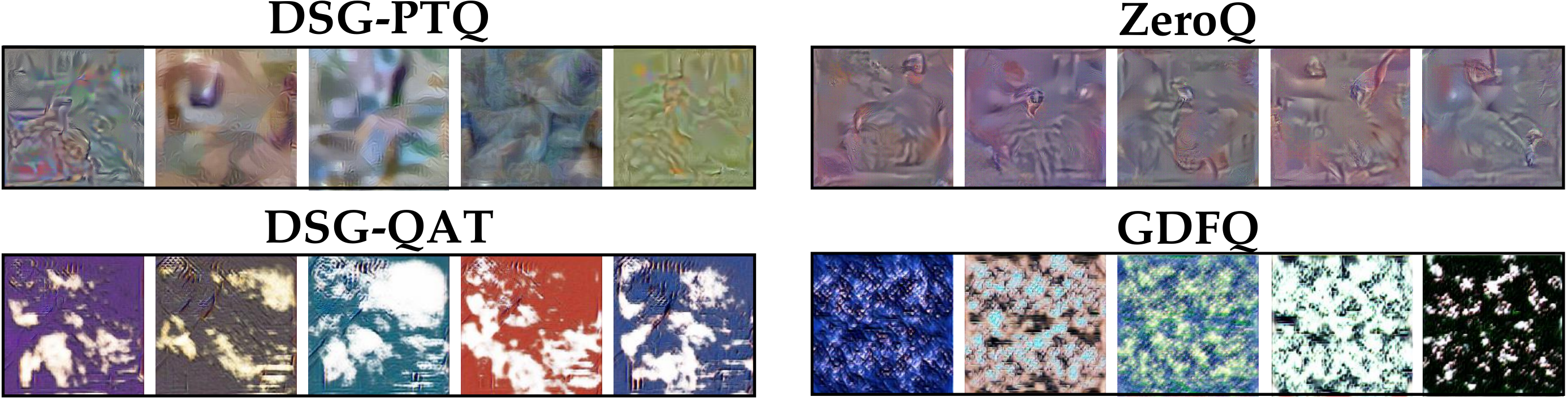}
\caption{Visualization of synthetic data for ZeroQ, GDFQ, and our DSG method in PTQ and QAT scenario. Each showcases 5 examples.}
\label{fig:visualization}
\end{figure*}

We visualize some synthetic samples of our method and other generative data-free quantization methods (ZeroQ, GDFQ). Take a closer look at Fig.~\ref{fig:visualization}, comparing pictures of DSG-PTQ and ZeroQ above, both of which generate images directly and update them iteratively. It is obvious that pictures generated by our DSG-PTQ method seem to have diverse colors with fine-grained textures and coarse-grained figures, while pictures of ZeroQ seem identical and have little difference in between. Meanwhile, below two sets of pictures are generated by DSG-QAT and GDFQ respectively. Owing to the generator, there are much more possibilities in the generation process and finally exhibit in both sets of images. It can be seen with naked eyes that samples of our DSG-QAT have distinct colors with higher saturation degrees, and white images inside the pictures have uncertain patterns. However, although samples generated by GDFQ have different colors and details inside the pictures as well, these samples have a uniform style in texture, which is tanglesome and poor in diversity. 

\subsubsection{Evaluation with Calibration Methods}

Additionally, to verify the versatility and robustness of the performance gains of our DSG method, we evaluate our DSG scheme with different calibration methods and compare it with other data generation methods under the same setting. The calibration methods include Percentile, EMA, and MSE. As shown in TABLE~\ref{tb:exp_quantization_method}, no matter which calibration methods is integrated, our DSG scheme substantially outstrips ZeroQ by 6.52\%, 9.36\%, and 0.61\%, respectively. The results strongly suggest that DSG can achieve the leading performance in various experimental settings, and thus the improvement is versatile and robust for various calibration methods. 

\begin{table}[t]
    \caption{Evaluation with calibration on ImageNet. We evaluation our DSG scheme with various calibration methods (Percentile, EMA, MSE) on ResNet18, and Vanilla means the calibration method adopted by ZeroQ, which simply obtain the quantizer by the maximum and minimum of the weight and activation.}
    \label{tb:exp_quantization_method}
	\centering
    \setlength{\tabcolsep}{2.mm}
    {\small
    \begin{tabular}{lccccc}
		\toprule
		{Method} &{No D}  &{W-bit}  &{A-bit} & {Quant}  &\tabincell{c}{{Top-1}}\\
		\midrule
		Baseline   &--  &-- &32 &32 &71.47  \\
		\midrule
		Real Data  &{\footnotesize\XSolidBrush}   &4  &4  & Vanilla &31.86  \\
		\midrule
		ZeroQ  &{\footnotesize\Checkmark}   &4  &4 & Vanilla  &26.04  \\
		{DSG} (Ours)  &{\footnotesize\Checkmark}   &4  &4 & Vanilla &\textbf{39.90}  \\
		\midrule
		Real Data  &{\footnotesize\XSolidBrush}   &4  &4  & Percentile &42.83  \\
		\midrule
		ZeroQ  &{\footnotesize\Checkmark}  &4  &4 &Percentile &32.24  \\
		{DSG} (Ours)  &{\footnotesize\Checkmark}  &4  &4 &Percentile &\textbf{38.76}  \\
		\midrule
		Real Data  &{\footnotesize\XSolidBrush}  &4  &4  & EMA & 42.67  \\
		\midrule
		ZeroQ  &{\footnotesize\Checkmark}  &4  &4  &EMA &32.31  \\
		{DSG} (Ours)  &{\footnotesize\Checkmark}  &4  &4 &EMA &\textbf{41.67}  \\
		\midrule
		Real Data  &{\footnotesize\XSolidBrush}  &4  &4  &MSE &41.45  \\
		\midrule
		ZeroQ  &{\footnotesize\Checkmark}   &4  &4  &MSE &39.39  \\
		{DSG} (Ours)   &{\footnotesize\Checkmark}  &4  &4 &MSE  &\textbf{40.00}  \\
		\bottomrule
	\end{tabular}
	}
\end{table}

\begin{table}[t]
    \caption{Evaluation with DFQ on ImageNet. We use cross-layer equalization and bias correction proposed by DFQ to perform per-layer quantization on ResNet18.}
    %\vspace{-0.14in}
    \label{tb:exp_dfq}
	\centering
	\setlength{\tabcolsep}{2.5mm}
    {\small
    \begin{tabular}{lccccc}
		\toprule
		{Method} &{No D}  &{No FT}  &{W-bit}  &{A-bit}  &\tabincell{c}{{Top-1}}\\
		\midrule
		Baseline   &--  &-- &32 &32 &69.76  \\
		\midrule
		Real Data  &{\footnotesize\XSolidBrush}   &{\footnotesize\Checkmark} &6  &6  &59.16  \\
		\midrule
		ZeroQ  &{\footnotesize\Checkmark}   &{\footnotesize\Checkmark} &6  &6  &58.12 \\
		{DSG} (Ours)  &{\footnotesize\Checkmark}   &{\footnotesize\Checkmark}  &6  &6 &\textbf{58.69}  \\
		\midrule
		Real Data  &{\footnotesize\XSolidBrush}   &{\footnotesize\Checkmark} &8  &8  &69.22  \\%
		\midrule
		ZeroQ  &{\footnotesize\Checkmark}   &{\footnotesize\Checkmark} &8  &8  &65.75  \\
		{DSG} (Ours)  &{\footnotesize\Checkmark}   &{\footnotesize\Checkmark}  &8  &8 &\textbf{68.88}  \\
		\bottomrule
	\end{tabular}
	}
\end{table}

We also evaluate the synthetic data with DFQ~\cite{Nagel_2019_ICCV}, which is a calibration method for data-free quantization. 
Specifically, DFQ has proposed cross-layer range equalization to equalize the different channel ranges of weight in per-layer quantization and bias correction which is to eliminate the biased quantization error. Both of the two techniques rely on the statistics of BN layers following the convolution layer. 
Therefore, BN layers are needed to calibrate the corresponding activations, so they have to proceed behind each convolution layer, which results in DFQ only working on specific network architectures and cannot be commonly practiced. 
Fortunately, generative methods, such as ZeroQ and our DSG, can work on arbitrary architectures, and the statistics of activations can take the place of BN statistics and thus be used in the DFQ method. 
TABLE~\ref{tb:exp_dfq} shows the closeups of two generative data-free quantization methods, \textit{i.e.}, ZeroQ and DSG, in conjunction with the DFQ method. Results show that our DSG outperforms ZeroQ by 0.57\% and 3.13\% in W6A6 and W8A8 cases.

\begin{table}[t]
    \caption{AdaRound on ImageNet with ResNet18 and MobileNetV2. We evaluate the DSG scheme on AdaRound, one of the SOTA methods of data-driven post-training quantization, which learns how to quantize weights using several batches of unlabeled samples. We adopt "Label"~\cite{9156818} and image prior ("Prior")~\cite{yin2020dreaming} techniques to evaluating our DSG scheme further.}
    \label{tb:adaround_res18}
	\centering
    \setlength{\tabcolsep}{0.4mm}
    {\small
    \begin{tabular}{llcccccc}
		\toprule
		Arch &{Method} &{No D}  &{Label} &{Prior}  &{W-bit}  &{A-bit}  &\tabincell{c}{{Top-1}}\\
		\midrule
		\multirow{24}{*}{ResNet18} &Real Data  &{\footnotesize\XSolidBrush} &{\footnotesize\XSolidBrush} 
		&{\footnotesize\XSolidBrush} &3  &32  & 64.16\\
		\cmidrule{2-8}
		&ZeroQ  &{\footnotesize\Checkmark} &{\footnotesize\XSolidBrush}  &{\footnotesize\XSolidBrush} &3  &32  & 49.86 \\
		&DSG (Ours)  &{\footnotesize\Checkmark} &{\footnotesize\XSolidBrush}  &{\footnotesize\XSolidBrush} &3  &32  & \textbf{56.09} \\
		&DSG (Ours)  &{\footnotesize\Checkmark} &{\footnotesize\Checkmark}  &{\footnotesize\XSolidBrush} &3  &32  & \textbf{58.27} \\
		&DSG (Ours)  &{\footnotesize\Checkmark} &{\footnotesize\Checkmark}  &{\footnotesize\Checkmark} &3  &32  & \textbf{61.32} \\
		\cmidrule{2-8}
		&Real Data  &{\footnotesize\XSolidBrush} &{\footnotesize\XSolidBrush} 
		&{\footnotesize\XSolidBrush} &4  &32  & 68.42 \\
		\cmidrule{2-8}
		&ZeroQ  &{\footnotesize\Checkmark} &{\footnotesize\XSolidBrush}  &{\footnotesize\XSolidBrush} &4  &32  & {63.86} \\
		&DSG (Ours)  &{\footnotesize\Checkmark} &{\footnotesize\XSolidBrush}  &{\footnotesize\XSolidBrush} &4  &32  & \textbf{66.87} \\
		&DSG (Ours)  &{\footnotesize\Checkmark} &{\footnotesize\Checkmark}  &{\footnotesize\XSolidBrush} &4  &32  & \textbf{67.09} \\
		&DSG (Ours)  &{\footnotesize\Checkmark} &{\footnotesize\Checkmark}  &{\footnotesize\Checkmark} &4  &32  & \textbf{67.78} \\
		\cmidrule{2-8}
		&Real Data  &{\footnotesize\XSolidBrush} &{\footnotesize\XSolidBrush} 
		&{\footnotesize\XSolidBrush} &5  &32  & 69.21\\
		\cmidrule{2-8}
		&ZeroQ  &{\footnotesize\Checkmark} &{\footnotesize\XSolidBrush}  &{\footnotesize\XSolidBrush} &5  &32  & 68.39 \\
		&DSG (Ours)  &{\footnotesize\Checkmark} &{\footnotesize\XSolidBrush}  &{\footnotesize\XSolidBrush} &5  &32  & \textbf{68.97} \\
		&DSG (Ours)  &{\footnotesize\Checkmark} &{\footnotesize\Checkmark}  &{\footnotesize\XSolidBrush} &5  &32  & \textbf{69.02} \\
		&DSG (Ours)  &{\footnotesize\Checkmark} &{\footnotesize\Checkmark}  &{\footnotesize\Checkmark} &5  &32  & \textbf{69.16} \\
		\cmidrule{2-8}
		&Real Data  &{\footnotesize\XSolidBrush} &{\footnotesize\XSolidBrush} 
		&{\footnotesize\XSolidBrush} &4  &8  &68.24 \\
		\cmidrule{2-8}
		&ZeroQ  &{\footnotesize\Checkmark} &{\footnotesize\XSolidBrush}  &{\footnotesize\XSolidBrush} &4  &8  &56.34  \\
		&DSG (Ours)  &{\footnotesize\Checkmark} &{\footnotesize\Checkmark}  &{\footnotesize\Checkmark} &4  &8  & \textbf{62.40} \\
		\midrule
		\multirow{8}{*}{MobileNetV2} &Real Data  &{\footnotesize\XSolidBrush} &{\footnotesize\XSolidBrush} 
		&{\footnotesize\XSolidBrush} &3  &32  &58.13 \\
		\cmidrule{2-8}
		&ZeroQ  &{\footnotesize\Checkmark} &{\footnotesize\XSolidBrush}  &{\footnotesize\XSolidBrush} &3  &32  & 11.07 \\
		&DSG (Ours)  &{\footnotesize\Checkmark} &{\footnotesize\Checkmark}  &{\footnotesize\Checkmark} &3  &32  & \textbf{45.40} \\
		\cmidrule{2-8}
		&Real Data  &{\footnotesize\XSolidBrush} &{\footnotesize\XSolidBrush} 
		&{\footnotesize\XSolidBrush} &4  &32  & 68.37\\
		\cmidrule{2-8}
		&ZeroQ  &{\footnotesize\Checkmark} &{\footnotesize\XSolidBrush}  &{\footnotesize\XSolidBrush} &4  &32  &56.16  \\
		&DSG (Ours)  &{\footnotesize\Checkmark} &{\footnotesize\Checkmark}  &{\footnotesize\Checkmark} &4  &32  & \textbf{58.13} \\
		\bottomrule
	\end{tabular}
	}
\end{table}

\subsubsection{Evaluation with Data-driven Quantization Methods}

Experiments above are conducted on calibration methods that optimize the clipping value for activations. We further evaluate our DSG data with a novel data-driven PTQ method named Adaround~\cite{nagel2020down}, which utilizes a rounding approach to quantize weights. Meanwhile, we introduce two other data generation tricks into our DSG scheme, \textit{e.g.}, generating data with labels~\cite{9156818} provides class information from parameters, image prior~\cite{yin2020dreaming} avoids generating unpractical scenes or unrecognizable patterns. We have conducted different bit-width for both weights and activations on different network architectures including ResNet18 and MobileNetV2 (TABLE~\ref{tb:adaround_res18}). And we generate 1024 samples in every single experiment. TABLE~\ref{tb:adaround_res18} shows that our DSG scheme outperforms ZeroQ by a wide margin when solely quantizing weights and preserving full-precision for activations. It is notable that in ultra-low bit-width settings (\textit{i.e.}, 3-bit), our DSG surpasses ZeroQ by 11.46\% on ResNet18 and a surprising 34.33\% on MobileNetV2. We also tried quantizing activations to 8-bit, and the results prove that our DSG is more robust to quantization of parameters, which outstrips ZeroQ by 6.06\% with ResNet18 on ImageNet.
Besides, we have the observation that the model gains further improvements when these tricks are applied together, even close the performance applying real data. Because our DSG is orthogonal with labels and image prior methods, so these methods can jointly boost the accuracy performance without any inconsistency. 

\section{Conclusion}

In this paper, we first revisit the sample generation process in generative data-free PTQ and QAT quantization and then give a theoretical analysis that the diversity of synthetic samples is crucial for the data-free quantization and reveal the homogenization of synthetic data in the distribution and sample levels.
In this paper, we propose a novel Diverse Sample Generation (DSG) scheme for generative data-free quantization, to address the deficiencies of previous methods which severely debases the quality of the synthetic data and further harms the performance of the quantized network. Our scheme has been evaluated on a variety of bit widths and neural architectures, and the results forcefully demonstrate the effectiveness and versatility of the DSG scheme. It shows notable accuracy improvements in ultra-low bit-width cases (e.g. W4A4). 
Moreover, benefiting from the enhanced diversity, the performance of the network is significantly improved in various methods integrating with synthetic data, which demonstrates that diversity is an important property of high-quality synthetic data.
We hope our work can provide directions for future research on data-free quantization.

\ifCLASSOPTIONcompsoc
  \section*{Acknowledgments}
\else
  \section*{Acknowledgment}
\fi

This work was supported by National Natural Science Foundation of China (62022009, 61872021), Beijing Nova Program of Science and Technology (Z191100001119050), State Key Lab of Software Development Environment (SKLSDE-2020ZX-06).

{%\small
\bibliography{egbib}
\bibliographystyle{ieee_fullname}
}

\clearpage

\appendices

\section{Main Proofs and Discussion}

\subsection{Proof of Lemma 1}
\label{app:lem1}
\begin{lem}
For any input domains $\mathcal{X}$ that includes multiples classes (at least 2) of samples, it can be modeled as several independent high-density $\{\mathcal{R}_{H1}, \cdots, \mathcal{R}_{Hh}\}$ and low-density $\{\mathcal{R}_{L1}, \cdots, \mathcal{R}_{Ll}\}$ sub-regions divided by possible decision surfaces, where $h\geq 1$ and $l\geq 0$.
\end{lem}
%证明（反证）：如果不存在低密度区域（存在且只存在1个高密度区域），根据平滑假设，决策面不能穿过任何地方，整个输入域只有一类；除此以外，决策面可以随意穿越输入区域，使其分割成多个整数个高密度或低密度区域，决策面划分类别>=标签类别。

\begin{proof}
As discussed in Section~\ref{sec:why}, the given input domain $\mathcal{X}$ follows \textit{Low-density assumption} and \textit{Smoothness assumption}.
And we use the method of proof by contradiction. 
If Lemma~\ref{lem:app1} is false, then at least one of the following is true:
\begin{itemize}
\item Low-density sub-regions in the input domain do not exist;
\item High-density sub-regions in the input domain do not exist or are only one.
\end{itemize}
And below we prove that none of the above is true.

(1) Let us assume that the first term holds.

When there is no low-density area, according to the low-density assumption, the possible decision surface can only pass through the low-density area, so there is no possible decision surface in the input domain, that is, the input domain contains and contains only one meaningful category. This violates the premise that the input domain contains at least two classes.

(2) Then we assume that the second term holds.

When high-density regions do not exist in the input domain, the entire input domain is a low-density region. According to the smoothness assumption, since there is no high-density path between any two samples in the input domain, they cannot be classified into the same category by the decision surface.
Therefore, any two samples in the input domain belong to different classes, that is, there is no determinable class division. This in principle violates the premise that the input domain contains multiple classes, that is, the input domain should be classifiable rather than distinct everywhere.

When there is only one high-density region in the input domain, also according to the smoothness assumption, only two samples in this high-density region can form a high-density path, i.e. belong to the same class. In other words, the input domain contains one and only one meaningful class, which also violates the premise that the input domain contains at least two classes.

So, to sum up, Lemma~\ref{lem:app1} must be true.

\end{proof}

\subsection{Proof of Theorem 1}
\label{app:th1}
\begin{thm}
%A given input domain $\mathcal{X}$ of scale $V$ consists of several sub-regions $\{R_0, R_1, \cdots\}$ of scales $\{V_0, V_1, \cdots\}$.
Given a set of all possible input domains $\mathbf{X}=\{\mathcal{X}_0, \mathcal{X}_1, \cdots \}$, whose $i$-th element can be denoted as $\mathcal{X}_i$ with scale $V^i$ and consists of several sub-regions $\{\mathcal{R}^i_1, \cdots, \mathcal{R}^i_{K^i}\}$ with scales $\{V^i_1, \cdots, V^i_{K^i}\}$, and the number $K^i\geq 2$ is unknown yet limited. Consider a sample set $\mathbf{x}^s=\{x_0^s,\cdots , x_N^s\} \subset \mathcal{X}^*$, where $\mathcal{X}^*=\mathbb{E}(\mathbf{X})$ denotes the potential input domain and the differences inside each sub-region of $\mathcal{X}$ is neglected. When the set $\mathbf{x}^s$ satisfies that for $\forall x_i^s\in\mathbf{x}^s$, $p (x_i^s\in \mathcal{R}^*_j)=\frac{V^*_j}{V^*}$, the information reflecting from all possible input domains $\mathbf{X}$ by the sample set $\mathbf{x}^s$ will be the maximized in mathematical expectation, where $V^*=\sum_{k=0}^{K^*}V^*_k$. 
%换成密度*体积表述
%因为内部距离被忽略，根据平滑假设密度无意义
\end{thm}

\begin{proof}

First, we discuss the properties of the potential input field $X^*=E(X)$. For the set of all possible input domains $\mathbf{X}=\{\mathcal{X}_0, \mathcal{X}_1, \cdots \}$, since the consistent modeling among all its possible elements (as shown in Lemma 1), so for the defined potential input domain $X^*$, the number of its sub-regions $K^*$ is $K^*=\max(K^0, K^1, \cdots)$.

For any $i$-th input domain $X_i$, the $K^*>K^i, {R_{K^i+1}\cdots R_{K^*}}$ can be regarded as empty sub-regions, and there is no specific ordering among all its sub-regions ${R_{0}, R_{1}\cdots R_{K^*}}$. For the potential input domain $X^*$, the mathematical expectation of the scale of any sub-region $R^*_j$ is $V^*_j=\mathbb{E}_i\mathbb{E}_j(V^i_j)$, so $\forall i, j, V^*_i = V^*_j$, \textit{i.e.}, $\forall j, V^*_j = V^*/K^*$. And for $X^*=\mathbb{E}(\mathbf{X})$, since the density properties of each possible input domain are completely random, $X^*$ can be seen as uniform in density in expectation.

Therefore, consider the set $\mathbf{x}^s=\{x_0^s,\cdots , x_N^s\}$ sampled from $\mathcal{X}^*$, for any sample $x_i^s$, the probability that it belongs to the $j$-th sub-region $\mathcal{R}^*_j$ is $p (x_i^s\in \mathcal{R}^*_j)$. Since differences within sub-regions are ignored, samples in the same sub-region can be considered to have the exact same class. Therefore, maximizing the information amount of the sample set $\mathbf{x}^s$ to reflect the latent region $\mathcal{X}^*$ can be expressed as:
\begin{align}
\max\mathcal{H}(\mathbf{x}^s)=-\sum_i\sum_j p (x_i^s\in \mathcal{R}^*_j)\log p (x_i^s\in \mathcal{ R}^*_j)
\end{align}
Since the sampling of all samples is completely independent, the above formula is equivalent to
\begin{align}
\sum_i\max\mathcal{H}_i(\mathbf{x}^s)=\sum_i\left(\max-\sum_j p (x_i^s\in \mathcal{R}^*_j)\log p (x_i^s \in \mathcal{R}^*_j)\right).
\end{align}

Since the samples are non-specific, we just need to discuss $\max\mathcal{H}_i(\mathbf{x}^s)$ here.
We simplify $p (x_i^s\in \mathcal{R}^*_j)$ as $p_j$, and the optimization problem can be defined as:
\begin{align}
\max\mathcal{H}_i(\mathbf{x}^s)=\max-\sum_{j=1}^{K^*} p_j\log p_j.
\end{align}
We introduce Lagrangian multiplier $\lambda$, the constructed Lagrangian function is:
\begin{align}
L\left(p_j, \lambda\right)=-\sum_{j=1}^{K^*} p_j \log p_j+\lambda\left(\sum_{j=1}^{K^*} p_j-1\right)
\end{align}
And then solve as:
\begin{align}
\left\{\begin{array}{r}
\frac{\partial L\left(p_j, \lambda\right)}{\partial p_j}=0 \\
\sum\limits_{j=1}^{K^*} p_j=1
\end{array}\right.
\end{align}

Then it is available that:
\begin{align}
&\frac{\partial L\left(p_j, \lambda\right)}{\partial p_j}=0 \\
\Rightarrow &\frac{\partial \left[-\left(p_j\log p_j\sum_{i\neq j}p_i\log p_i\right)+\lambda\left(p_j+\sum_{i\neq j}p_i-1\right)\right]}{\partial p_i}
=0 \\
\Rightarrow &-\left(\log p_j+1\right)+\lambda=0 \\
\Rightarrow &p_j=\frac{1}{K^*}.
\end{align}
Put $p_j=2^{\lambda-1}$ into $\sum_{j=1}^{K^*} p_j=1$ to get:
\begin{align}
\sum_{i=1}^n 2^{\lambda-1}=1 \Rightarrow 2^{\lambda-1}=\frac{1}{n} \Rightarrow p_i=\frac{1}{n}.
\end{align}
Bring $p_i=\frac{1}{n}$ into $H(X)=-\sum_{i=1}^n p_i \log p_i$ to get $H(X)=\log n$, so $ H(X) \leq \log n$;
Therefore, when $p_j=\frac{1}{K^*}$, the information entropy $\mathcal{H}(\mathbf{x}^s)$ is maximized.

Since according to the above proof, $\forall j, V^*_j = \frac{V^*}{K^*}$, so when
When $p (x_i^s\in \mathcal{R}^*_j)=\frac{V^*_j}{V^*}=1/K^*$, the information entropy $\mathcal{H}(\mathbf{x}^s)$ is maximized, that is, the sample set $\mathbf{x}^s$ can maximize the amount of information that reflects the potential area $\mathcal{X}^*$ at this time. And Theorem~\ref{th:dfq} is proved.

\end{proof}

\section{Visualization}

In Fig.~\ref{fig:app-vis} we show the visualization of more synthetic samples. Compared to the synthetic samples in existing generative data-free methods like ZeroQ and GDFQ, the synthetic samples produced by our DSG have more diverse colors and textures in visualization, which verifies from another aspect that DSG improves synthetic samples through diversification.

\begin{figure}[!t]
\centering
\includegraphics[width=1.0\linewidth]{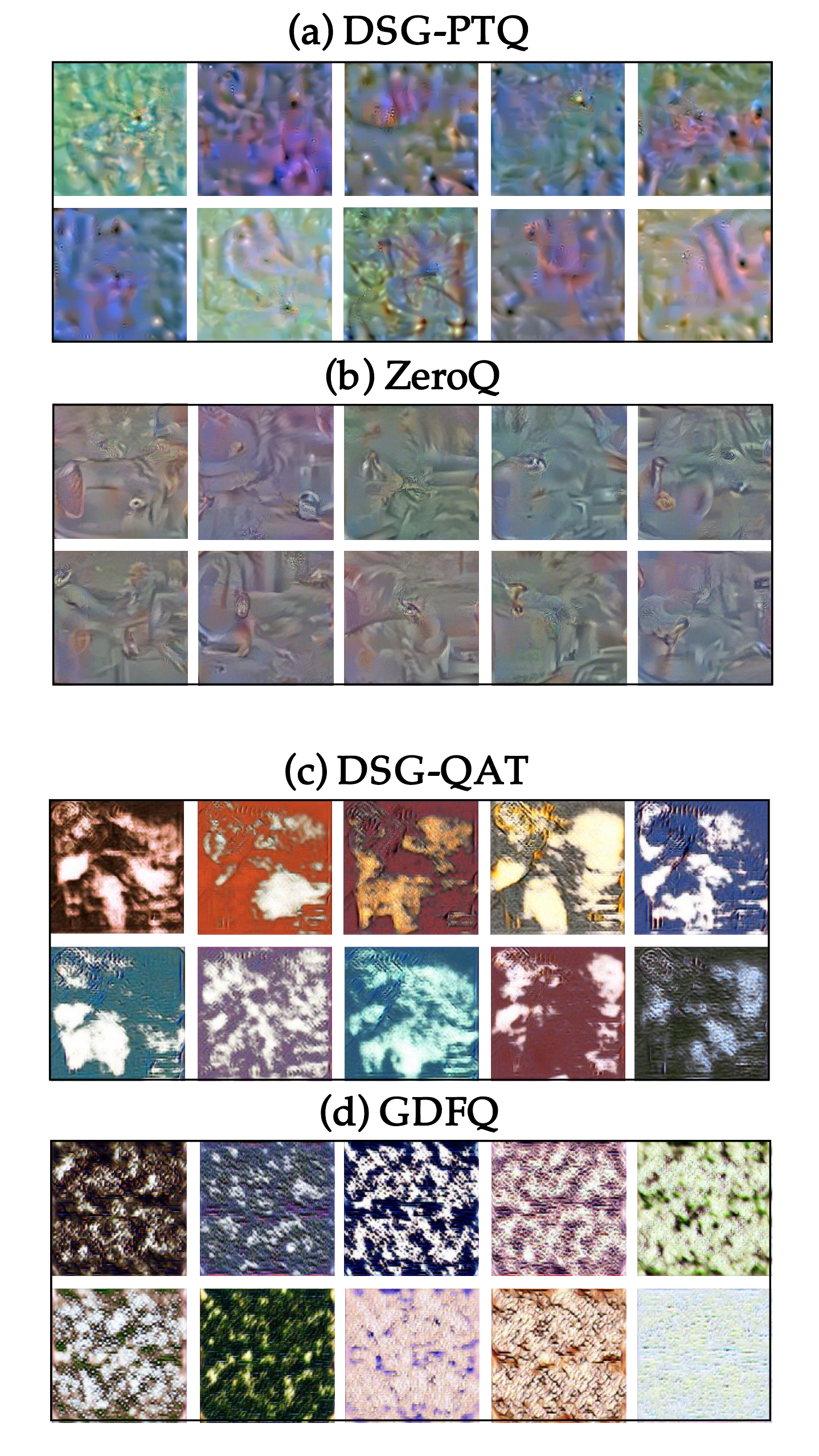}
% \vspace{-0.2in}
\caption{Visualization of more synthetic data for ZeroQ, GDFQ, and our DSG method in PTQ and QAT approaches. Each showcases 10 examples.}
\label{fig:app-vis}
\end{figure}

% or


\end{document}